\documentclass[runningheads]{llncs}
\usepackage{amsmath}
\usepackage{amssymb}
\usepackage{stmaryrd}
\usepackage[table,xcdraw]{xcolor}
\usepackage{bm}
\usepackage{graphicx}
\usepackage{subfigure}
\usepackage{xspace}
\usepackage{todonotes}
\usepackage{booktabs}
\usepackage{comment}
\usepackage{mathrsfs}
\usepackage{enumerate}
\usepackage{hyperref}
\usepackage{float}
\usepackage[ruled]{algorithm2e}
\newcommand \domB {\mathbb{B}}
\newcommand \domN {\mathbb{N}}
\newcommand \domR {\mathbb{R}}

\newcommand \T {\mathrm{T}}
\newcommand \V {\mathrm{V}}

\newcommand \Coloneqq {::=}

\newcommand \network[1] {\mathcal{#1}}
\newcommand \automaton[1] {\mathcal{#1}}
\newcommand \tableau[1] {\mathcal{#1}}
\newcommand \lang[1] {\mathscr{L}({#1})}

\newcommand \ltl {LTL\xspace}
\newcommand \ltlf {LTL$_f$\xspace}

\newcommand \bltl {BLTL\xspace}

\newcommand \nnf {NNF\xspace}
\newcommand \dnn {DNN\xspace}
\newcommand \qnn {QNN\xspace}
\newcommand \bnn {BNN\xspace}
\newcommand \nn  {NN\xspace}
\newcommand \idl {IDL\xspace}

\newcommand \pspace {\comsps{PSPACE}\xspace}
\newcommand \cnf {CNF\xspace}

\newcommand \mnist {MNIST\xspace}
\newcommand \uci {UCI Adult\xspace}
\newcommand \mse {MSE\xspace}

\newcommand \gpt {GPT-4\xspace}

\newcommand \pac {PAC\xspace}
\newcommand \smt {SMT\xspace}
\newcommand \sat {SAT\xspace}
\newcommand \bdd {BDD\xspace}

\newcommand \linlayer {LIN\xspace}
\newcommand \bnlayer {BN\xspace}
\newcommand \binlayer {BIN\xspace}

\newcommand \acc {Acc\xspace}
\newcommand \accplus {Acc+\xspace}
\newcommand \fair {Fair\xspace}
\newcommand \fairplus {Fair+\xspace}
\newcommand \fairnum {\emph{fair num}\xspace}
\newcommand \sizeoftestset {\emph{test size}\xspace}
\newcommand \as {ASR\xspace}

\newcommand \opX {\mathsf{X}}
\newcommand \opXd {\overline{\mathsf{X}}}
\newcommand \opU {\mathsf{U}}
\newcommand \opG {\mathsf{G}}
\newcommand \opF {\mathsf{F}}
\newcommand \opR {\mathsf{R}}
\newcommand \opN {\rhd}

\newcommand \semantics[1] {\llbracket{#1}\rrbracket}
\newcommand \defeq {\stackrel{\tiny{\rm def}}{=}}

\newcommand \tname[1] {(\textsc{#1})}
\newcommand \trule[4] {\begin{array}{c|c} {#1} & {#2} \\ \hline {#3} & {#4} \end{array}}
\newcommand \comsps[1] {\textbf{#1}}

\DeclareMathOperator* \bin {{\sf bin}}
\DeclareMathOperator* \dec {{\sf dec}}  
 
\DeclareMathOperator* \cons  {{\sf Cons}}
\DeclareMathOperator* \cl  {{\sf Cl}}
\DeclareMathOperator* \sub  {{\sf Sub}}
\DeclareMathOperator* \slen  {{\sf slen}}
\DeclareMathOperator* \len  {{\sf len}}

\DeclareMathOperator* \dom  {{\sf dom}}
\DeclareMathOperator* \ran  {{\sf ran}}

\iftrue
\newcommand{\wl}[1]{\todo[inline,color=red!10,caption={WL}]{\textbf{Liu:} #1}}
\newcommand{\fs}[1]{\todo[inline,color=blue!10,caption={FS}]{\textbf{Fu:} #1}}
\newcommand{\yt}[1]{\todo[inline,color=green!10,caption={YT}]{\textbf{Ye:} #1}}
\else
\newcommand{\wl}[1]{}
\newcommand{\fs}[1]{}
\newcommand{\yt}[1]{}
\fi

\begin{document}

\title{An Automata-Theoretic Approach to Synthesizing Binarized Neural Networks}

\author{
    Ye Tao\inst{1}\orcidID{0009-0007-1478-9144}\and
    Wanwei Liu\inst{1}\thanks{Corresponding Author}\orcidID{0000-0002-2315-1704}
    \and
    Fu Song \inst{2,3,4}\orcidID{0000-0002-0581-2679}\and
    Zhen Liang \inst{5} \orcidID{0000-0002-1171-7061}\and
    Ji Wang \inst{5} \orcidID{0000-0003-0637-8744}\and
    Hongxu Zhu \inst{1}
}
\authorrunning{Y. Tao et al.}
            
\institute{
    College of Computer Science and Technology, National University of Defense Technology\\
    \email{\{taoye0117,wwliu,zhuhongxu\}@nudt.edu.cn} \and
   School of Information Science and Technology, ShanghaiTech University\\
    \email{songfu@shanghaitech.edu.cn} \and
Institute of Software, Chinese Academy of Sciences \& University of Chinese Academy of
Sciences  \and
Automotive Software Innovation Center\and
   Institute for Quantum Information \& State Key Laboratory for High Performance Computing, National University of Defense Technology \\
    \email{\{liangzhen,wj\}@nudt.edu.cn}\\
}

\maketitle
    \begin{abstract}    
    Deep neural networks, ({\dnn}s, a.k.a. {\nn}s), have been widely used in various tasks and have been proven to be successful. However, the accompanied expensive computing and storage costs make the deployments in resource-constrained devices a significant concern. To solve this issue, quantization has emerged as an effective way to reduce the costs of {\dnn}s with little accuracy degradation by quantizing floating-point numbers to low-width fixed-point representations. Quantized neural networks ({\qnn}s) have been developed, with binarized neural networks ({\bnn}s) restricted to binary values as a special case. Another concern about neural networks is their vulnerability and lack of interpretability.  Despite the active research on trustworthy of {\dnn}s, few approaches have been proposed to {\qnn}s. To this end, this paper presents an automata-theoretic approach to synthesizing {\bnn}s that meet designated properties. More specifically, we define a temporal logic, called \bltl, as the specification language. We show that each \bltl formula can be transformed into an automaton on finite words. To deal with the state-explosion problem, we provide a tableau-based approach in real implementation.
    For the synthesis procedure, we utilize \smt solvers to detect the existence of a model (i.e., a \bnn)
    in the construction process.
    Notably, synthesis provides a way to determine the hyper-parameters of the network before training. Moreover, we experimentally evaluate our approach and demonstrate its effectiveness in improving the individual fairness and local robustness of {\bnn}s while maintaining accuracy to a great extent.
    \end{abstract}
%   
%   
%
    
	%-----------------------------------------------
	
	%%
	\section{Introduction}
	\label{Sec: Intro}

    Deep Neural Networks ({\dnn}s) are increasingly used in a variety of applications,
    from image recognition to autonomous driving, due to their high accuracy in classification and prediction tasks \cite{simonyan2014very,tesla}.
    However, two critical challenges emerge, high-cost and a lack of trustworthiness, that impede their further development.

    On the one hand, a modern \dnn typically contains a large number of parameters which are typically stored as $32$-bit floating-point numbers (e.g., \gpt contains about 100 trillion parameters~\cite{chatgpt}), thus
    an inference often demands more than a billion floating-point operations.
    As a result, deploying a modern \dnn requires huge computing and storage resources,
    thus it is challenging for resource-constrained embedding devices.
    To tackle this issue, quantization has been introduced, which compresses a network by converting floating-point numbers to low-width fixed-point representations,
    so that it can significantly reduce both memory and computing costs using fixed-point arithmetic with a relatively small side-effect on the network’s accuracy~\cite{nagel2021white}.

    On the other hand, neural networks are known to be vulnerable to input perturbations, namely, slight input disturbance may dramatically change their output \cite{eykholt2018robust,BuZDS22,SongLCFL21,ChenCFDZSL21,ZCWYSS21,ChenZZS23,chen2022towards,CZSCFL22}.
    In addition, {\nn}s are often treated as black box~\cite{huang2020survey},
    and we are truly dearth of understanding of the decision-making process inside the ``box''. As a result, a natural concern is whether {\nn}s can be trustworthy,
    especially in some safety-critical scenarios,
    where erroneous behaviors might lead to serious consequences.
    One promising way to tackle this problem is formal verification,
    which defines properties that we expect the network to satisfy and rigorously checks whether the network meets our expectations.
    Numerous verification approaches have been proposed recently aiming at this purpose \cite{huang2020survey}. Nevertheless, these approaches in general ignore rounding errors in quantized computations, making them unable to apply for quantized neural networks ({\qnn}s).
    It has been demonstrated that specifications that hold for a floating-point numbered {\dnn} may not necessarily hold after quantizing the inputs and/or parameters of the {\dnn} \cite{BuZDS22,giacobbe2020many}.
    For instance, a {\dnn} that is robust to given input perturbations might become non-robust after quantization.
    Compared to \dnn verification~\cite{huang2020survey,li2019analyzing,liu2021algorithms,liu2020verifying,10.1007/978-3-031-35257-7_15,ZhaoZCSCL22,GuoWZZSW21}, verifying \qnn is truly a more challenging and less explored problem.
    Evidences show that the verification problem for {\qnn}s is harder than {\dnn}s \cite{henzinger2021scalable},
    and only few works are specialized for verifying {\qnn}s~\cite{baluta2019quantitative,cheng2018verification,giacobbe2020many,henzinger2021scalable,narodytska2018verifying,shih2019verifying,zhang2021bdd4bnn,ZZCSC23,zhang2022qvip,ZhangSS23}.

    %our work%
	%-----------------------------------------------

     In this paper, we concentrate on {\bnn}s (i.e., binarized neural networks), a special type of {\qnn}.
     Although formal verification has been the primary explored approach to verifying (quantized) neural networks,
     we pursue another promising line, synthesizing the expected binarized neural networks directly.
     In other words, we aim to construct a neural network that satisfies the expected properties we specify, rather than verifying an existing network's compliance with those properties. 
    To achieve this, we first propose, \bltl, an extension of \ltlf (namely, \ltl defined on finite words),
    as the specification language. This logic can conveniently describe data-related properties of {\bnn}s.
    We then provide an approach to converting a \bltl formula to  an equivalent automaton.
    The synthesis task is then boiled down to find a path from an initial state to an accepting state in the automaton.

    Unfortunately, such a method suffers from the state-exploration problem.
    To mitigate this issue, we observe
    that it is not necessary to synthesize the entire \bnn since the desired properties are only related to some specific hyper-parameters of the network.
    To this end, we propose a tableau-based approach:
    To judge whether a path is successfully detected, we check the satisfiability of the associated \bltl formulas, and convert the problem into an \idl-solving problem, which can be efficiently solved.
    Besides, we prove the existence of a tracing-back threshold, which allows us to do backtracking earlier to avoid doing trace searching that is unlikely to lead to a solution.
    The solution given by the solver provides the hyper-parameters of the \bnn, including the length of the network and crucial input-output relations of blocks. Afterwards, one can perform a block-wise training to obtain a desired \bnn.

    We implement a prototype synthesizing tool and evaluate our approach on local robustness and individual fairness.
    The experiments demonstrate that our approach can effectively improve the network's reliability compared to the baseline, especially for individual fairness.

    The main contributions of this work are summarized as follows:
    \begin{itemize}
        \item We present a new temporal logic, called \bltl, for describing properties of {\bnn}s, 
                and provide an approach to transforming \bltl formulas into equivalent finite-state automata.
        \item We propose an automata-theoretic synthesis approach that determines the hyper-parameters of a \bnn model before training.
        \item We implement a prototype synthesis tool and evaluate the effectiveness on two concerning properties, demonstrating the feasibility of our method.
    \end{itemize}

 \noindent {\bf Related Work.}
     For {\bnn}s, several verification approaches have been proposed. Earlier work reduces the {\bnn} verification problem to hardware verification (i.e., verifying combinatorial circuits),
     for which \sat solvers are harnessed~\cite{cheng2018verification}.
     Following this line, \cite{narodytska2018verifying} proposes a direct encoding from the \bnn verification problem into the \sat problem. 
    \cite{narodytska2020search} studies 
    the effect of \bnn architectures on the performance of \sat solvers and 
     uses this information to train \sat-friendly {\bnn}s.
     \cite{baluta2019quantitative} provides a framework for approximately quantitative verification of {\bnn}s with \pac-style guarantees via approximate \sat model counting.
     Another line of \bnn verification encodes a \bnn and its input region 
     into a binary decision diagram (\bdd), 
     and then one can verify some properties of the network by analyzing {\bdd}.
     \cite{shih2019verifying} proposes an Angluin-style learning algorithm to compile a \bnn on a given input region into a \bdd, and utilize a
     \sat solver as an equivalence oracle to query.
    \cite{zhang2021bdd4bnn} has developed a more efficient \bdd-based quantitative verification framework by exploiting the internal structure of {\bnn}s.
     Few work has been dedicated to \qnn verification so far.
     \cite{giacobbe2020many}
     shows that the properties guaranteed by the {\dnn} are not preserved after quantization.
     To resolve this issue, they introduce an approach to verifying {\qnn}s by using \smt solvers in bit-vector theory.
     Later, \cite{henzinger2021scalable} proves that verifying {\qnn} with bit-vector specifications is \pspace-Hard.
     More recently, \cite{zhang2022qvip,ZhangSS23} reduce the verification problem into integer linear constraint solving which are significantly more efficient than the SMT-based one.

 \smallskip
 \noindent {\bf Outline.} 
 The rest of the paper is organized as follows:
    In Section \ref{Sec: Prelim}, we introduce preliminaries. We present the specification language \bltl in Section \ref{Sec: BLTL}.
    In Section \ref{Sec: Automata_Cons}, we show how to translate
    a \bltl formula into an equivalent automaton, which is the basic
    of tableau-based approach for synthesis, and
    technical details are given in Section \ref{Sec: Synthesis}.
    The proposed approach is implemented and evaluated in Section \ref{Sec:Expr_Eval}.
    We conclude the paper in Section \ref{Sec: Con}.

\section{Preliminaries}
	\label{Sec: Prelim}
    We denote by $\domR$, $\domN$, and $\domB$ the set of real numbers, natural numbers, and Boolean domain $\{0,1\}$, respectively.
    We use $\domR^n$ and $\domB^n$ to denote the set of real number vectors and binary vectors with $n$ elements, respectively.
    For $n \in \domN$, let $[n]$ be the set $\{0, 1,2,\dots,n-1\}$.
    We will interchangeably
	use the terminologies $0$-$1$ vector and binary vector in this paper.
    For a binary vector $\bm{b}$, we use $\dec(\bm{b})$ to denote its corresponding decimal number, and conversely let $\bin(d)$ be the corresponding binary vector which encodes the
	number $d$. For example, let $\bm{b}=(0,1,1)^\T$, then we have $\dec(\bm{b})=3$. Note that $\bin(\dec(\bm{b}))=\bm{b}$ and $\dec(\bin(d))=d$.
	For two binary vectors $\bm{a}=\left(a_0,\ldots,a_{n-1}\right)^\T$ and $\bm{b}=\left(b_0,\ldots,b_{n-1}\right)^\T$
	with the same length, we denote by $\bm{a}\sim\bm{b}$ if $a_i\sim b_i$ for all $i\in [n]$, otherwise $\bm{a}\not\sim\bm{b}$,
        where $\sim\in\{>, \geq, <, \leq, =\}$.
        Note that $\bm{a}\neq\bm{b}$
        if $a_i\neq b_i$ for some $i\in [n]$.

    A (vectorized) Boolean function takes a $0$-$1$ vector as input and returns another $0$-$1$ vector.
    Hence, it is essentially a mapping from integers to integers when each $0$-$1$ vector $\bm{b}$ is viewed as an integer $\dec(\bm{b})$. We denote by $\bm{I}_n$ the identity function such that $\bm{I}_n \left(\bm{b}\right) = \bm{b}$,
    for any $\bm{b} \in \domB^n$, where the subscript $n$ may be dropped when it is clear from the context.
    We use \emph{composition} operation $\circ$ to represent the function composition among Boolean functions.

    A \emph{binarized neural network} (\bnn) is a feed-forward neural network, %, in which weights and activations are either $-1$ or $1$.
      composed of several internal blocks and one output block~\cite{shih2019verifying,zhang2021bdd4bnn}.
    \iffalse
    Internal blocks consists of a linear layer (\linlayer),
    a batch normalization layer (\bnlayer), and a binarization layer (\binlayer).
    \fi
    Each internal block is comprised of 3 layers and can be viewed as a mapping $f: \{-1, 1\}^n \to \{-1, 1\}^m$.
    Slightly different from internal blocks,
    the output block outputs the classification label to which the highest activation corresponds,
    thus, can be seen as a mapping
    ${\sf out}: \{-1, 1\}^n \to {\domR}^p$, where $p$ is the number of classification labels of the network.

    Since the binary values $-1$ and $+1$ can be represented as their Boolean counterparts $0$ and $1$ respectively,
    each internal block can be viewed as a Boolean function $f: \domB^n \to \domB^m$~\cite{zhang2021bdd4bnn}.
    Therefore, ignoring the slight difference in the output block,
    an $n$-block \bnn\ $\network{N}$ can be encoded via a series of Boolean functions
	$f_i: \domB^{\ell_i}\to\domB^{\ell_{i+1}}$ ($i=0, 1, \ldots, n-1$),
        and $\network{N}$ works as the combination
	of these Boolean functions, namely, it corresponds to the function,
	\[
	f_\network{N} = f_{n-1}\circ f_{n-2}\circ\cdots\circ f_1\circ f_0.
	\]

    \emph{Integer difference logic} (\idl) is a fragment of linear integer arithmetic, in which  atomic formulas must be of the form $x - y \sim c$
    where $x$ and $y$ are integer variables, and $c$ is an integer constant, $\sim \in \{\le,\ge,<,>,=, \neq\}$.
    All these atomic formulas 
    can be transformed into constraints of the form $x - y \leq c$~\cite{barrett2010smt}.
    For example, $x - y = c$ can be transformed into $x - y \leq c \wedge x - y \geq c$.

    The task of an \idl-problem is to check the satisfiability of an \idl formula in conjunctive normal form (\cnf)
    \[(x_1 - y_1 \le c_1) \wedge \dots \wedge (x_n - y_n \le c_n),\]
    which can be in general converted into the cycle detection problem in a weighted, directed graph with $O(n)$ nodes and $O(n)$ edges, and solved by e.g., Bellman-Ford or Dijkstra’s algorithm, in $O(n^2)$ time~\cite{idl}. 

	%-----------------------------------------------
	\section{The Temporal Logic \bltl}
	\label{Sec: BLTL}

    \subsection{Syntax and Semantics of \bltl}
    \label{SubSec: SandS}

	Let us fix a signature $\bm{\Sigma}$, consisting of a set of desired Boolean functions and $0$-$1$ vectors.
	Particularly, let $\bm{\Sigma}_{\V}$ be the subset of $\bm{\Sigma}$  containing only $0$-$1$ vectors.

	Terms of \bltl are described via BNF as follows:
	\[
	\bm{t} \Coloneqq \bm{b} \mid f \left(\bm{t}\right) \mid  \opN^k\bm{t}
	\]
	where $\bm{b}\in \bm{\Sigma_\V}$ is a $0$-$1$ vector, called \emph{vector constant},
	 $f\in\bm{\Sigma}\setminus \bm{\Sigma_\V}$ is a Boolean function, and $k\in\domN$ is a constant, and $\opN^{k}$ in $\opN^k\bm{t}$ denotes
$k$ placeholders for $k$ consecutive blocks of a \bnn (i.e., $k$ Boolean functions) to be applied onto the term $\bm{t}$. We remark that $\opN^0\bm{t}=\bm{t}$.

   %We say that a term $\bm{t}$ is \emph{closed} if it does not involve any vector variable.
	
	\bltl formulas are given via the following  grammar:
	\[
	\psi \Coloneqq \top\mid \bm{t}\sim\bm{t} \mid \neg\psi \mid \psi\vee\psi \mid \opX\psi \mid \psi\opU\psi  %\mid \forall \bm{x}\in\domB^k.\psi
	\]
 where $\sim \in \{\le,\ge,<,>,=\}$, $\opX$ is the \emph{Next} operator and $\opU$ is the \emph{Until} operator.

	We define the following derived Boolean operators, quantifiers with finite domain, and temporal operators:
 \[\begin{array}{ccc}
  \psi_1\wedge\psi_2\defeq\neg(\neg\psi_1\vee \neg\psi_2) &  \opF\psi\defeq\top\opU\psi   & \opG\psi\defeq\neg\opF\neg\psi \\
\psi_1\rightarrow\psi_2  \defeq (\neg \psi_1) \vee \psi_2 & \psi_1\opR\psi_2\defeq\neg(\neg\psi_1\opU\neg\psi_2)   & \opXd\psi\defeq\neg\opX\neg\psi \\
\forall \bm{x}\in \domB^k.\psi\defeq \bigwedge_{\bm{b}\in \domB^k\cap\bm{\Sigma_\V}}\psi[\bm{x}/\bm{b}] & \qquad \exists \bm{x}\in \domB^k.\psi \defeq \neg\forall \bm{x}\in \domB^k.\neg \psi
 \end{array}\]
 where $\psi[\bm{x}/\bm{b}]$ denotes the \bltl formula obtained from $\psi$ by replacing each occurrence of $\bm{x}$ with $\bm{b}$.

 The semantics of \bltl formulas is defined w.r.t. a \bnn $\network{N}$ given by
  the composition of Boolean functions $f_\network{N} = f_{n-1}\circ f_{n-2}\circ\cdots\circ f_1\circ f_0$, and a position $i\in\domN$.
	We first define the semantics of terms, which is given by the function $\semantics{\bullet}_{\network{N},i}$, inductively:
	\begin{itemize}
		\item $\semantics{\bm{b}}_{\network{N},i}=\bm{b}$ for each vector constant $\bm{b}$;
 \item $\semantics{f\left(\bm{t}\right)}_{\network{N},i}=f
            \left(\semantics{\bm{t}}_{\network{N},i}\right)$;

\item $\semantics{\opN^k\bm{t}}_{\network{N},i}= 
\left\{ \begin{array}{ll}
   (f_{i+\slen(\bm{t})+k-1}\circ \cdots \circ f_{i+\slen(\bm{t})})(\semantics{\bm{t}}_{\network{N},i}),  & \mbox{if } k\geq 1; \\
  \semantics{\bm{t}}_{\network{N},i},   &  \mbox{if } k=0;
\end{array}\right.$\\

where $f_i$ is the identity Boolean function $\bm{I}$ if $i\geq n$,  $\slen(\bm{b})=0$, $\slen(f(\bm{t}))=\slen(\bm{t})+1$
 and $\slen(\opN^k\bm{t})=\slen(\bm{t})+k$.
	\end{itemize}

 Note that we assume the widths of Boolean functions and their argument vectors are compatible.
	
	%--------------------
	\begin{proposition}
		\label{Lem:OpN_Power}
	We have: $\semantics{\opN^k\opN^{k'}\bm{t}}_{\network{N},i}=
		\semantics{\opN^{k+k'}\bm{t}}_{\network{N},i}$.	  	
  \end{proposition}

	Subsequently, the semantics of \bltl formulas is characterized via the \emph{satisfaction} relation $\models$, inductively:
	\begin{itemize}
		\item $\network{N}, i\models \top$ always holds;
		\item $\network{N}, i\models\bm{t}_1\sim\bm{t}_2$ iff $\semantics{\bm{t}_1}_{\network{N},i}\sim\semantics{\bm{t}_2}_{\network{N},i}$;
		\item $\network{N}, i\models\neg\varphi$ iff $\network{N},i\not\models\varphi$;
		\item $\network{N}, i\models\varphi_1\vee\varphi_2$ iff $\network{N},i\models\varphi_1$ or $\network{N},i\models\varphi_1$;
		\item $\network{N}, i\models\opX\psi$ iff $i<n-1$ and $\network{N}, i+1\models\psi$;
		\item $\network{N}, i\models\psi_1\opU\psi_2$ iff there is $j$ such that $i\leq j< n$, $\network{N},j\models\psi_2$
		and $\network{N},k\models\psi_1$ for each $i\leq k<j$;
	%	\item $\network{N},i\models\forall\bm{x}\in\domB^k.\psi$ iff
	%	$\network{N},i\models_{\mu[\bm{b}/\bm{x}]}\psi$ fo each $\bm{b}\in\domB^k$, where $\mu[\bm{b}/\bm{x}]$ denotes the update of the valuation $\mu$ by assigning the $0$-$1$ vector $\bm{b}$ to the vector variable $\bm{x}$.
	\end{itemize}
	
	We may write $\network{N}\models\psi$ in the case of $i=0$.
	In the sequel, we denote by $\lang{\psi}$ the set of {\bnn}s $\{\network{N}\mid \network{N}\models\varphi\}$ for each formula $\varphi$,
 and denote by $\psi_1\equiv\psi_2$ if $\network{N},i\models \psi_1
		\Leftrightarrow \network{N},i\models \psi_2$ for every \bnn $\network{N}$ and $i$.
	%and $E$.
	
	%
	\begin{proposition}
		\label{Lem:BLTL_Equiv}
		The following statements hold:
		\begin{enumerate}
			\item $\opG\psi\equiv\bot\opR\psi$;
			\item $\opF\psi\equiv\psi\vee\opX\opF\psi$;
			\item $\opG\psi\equiv\psi\wedge\opXd\opG\psi$;
			\item $\psi_1\opU\psi_2\equiv \psi_2\vee(\psi_1\wedge\opX(\psi_1\opU\psi_2))$;
			\item $\psi_1\opR\psi_2\equiv\psi_2\wedge(\psi_1\vee\opXd(\psi_1\opR\psi_2))$.
		\end{enumerate}
	\end{proposition}
	
	For a \bltl formula $\varphi$ and a \bnn $\network{N}$, the \emph{model checking} problem w.r.t.
	$\varphi$ and $\network{N}$ is to decide whether $\network{N}\models\varphi$ holds.
	
	%% NNF
	With the above derived operators, together with the patterns $\neg\neg\psi\equiv\psi$
 and $\neg(\bm{t}_1\sim\bm{t}_2)\equiv\bm{t}_1\not\sim\bm{t}_2$, \bltl formulas can be transformed into \emph{negation normal form} (\nnf) by pushing the negations ($\neg$) inward, till no the negations are involved.
% till each appears associated with atoms only \textemdash\
%	we sometimes abbreviate $\neg(\bm{t}_1\sim\bm{t}_2)$ as $\bm{t}_1\not\sim\bm{t}_2$.
	%We call formulas are in the  in this situation. 
%$\neg(\bm{t}_1=\bm{t}_2)\equiv\bm{t}_1\not=\bm{t}_2\not \equiv \bm{t}_1\neq\bm{t}_2$
 
	%\wl{A Hintikka-Style closure}
	Given two sets of formulas $\Gamma$ and $\Gamma'$ in \nnf, %and suppose that each belonging to  it is in \nnf,
we say that   $\Gamma'$ is a  \emph{proper closure} of $\Gamma$, if the following conditions hold:
	\begin{itemize}
		\item $\Gamma\subseteq\Gamma'$.
		\item $\psi_1\wedge\psi_2\in\Gamma'$ implies that both $\psi_1\in\Gamma'$ and
		$\psi_2\in\Gamma'$.
		\item $\psi_1\vee\psi_2\in\Gamma'$ implies that either $\psi_1\in\Gamma'$ or
		$\psi_2\in\Gamma'$.
		\item $\psi_1\opU\psi_2\in\Gamma'$ implies $\psi_2\vee(\psi_1\wedge\opX(\psi_1\opU\psi_2))\in\Gamma'$.
		\item $\psi_1\opR\psi_2\in\Gamma'$ implies $\psi_2\wedge(\psi_1\vee\opXd(\psi_1\opR\psi_2))\in\Gamma'$.
            \iffalse
		\item $\forall\bm{x}\in\domB^k.\psi\in\Gamma'$ implies $\psi^{\bm{x}}_{\bm{b}}\in\Gamma'$ for each $\bm{b}\in\domB^k\cap\bm{\Sigma}_\V$.	
		\item $\exists\bm{x}\in\domB^k.\psi\in\Gamma'$ implies $\psi^{\bm{x}}_{\bm{b}}\in\Gamma'$ for some $\bm{b}\in\domB^k\cap\bm{\Sigma}_\V$.
        \fi
	\end{itemize}

	We denote by $\cl(\Gamma)$ the set consisting of all proper closures of $\Gamma$ (note that $\cl(\Gamma)$ is a family of formula sets.)
	We also denote by $\sub(\psi)$ the set of the subformulas of $\psi$ except that

 \begin{itemize}
     \item if $\psi_1\opU\psi_2\in \sub(\psi)$, then
$\psi_2\vee(\psi_1\wedge\opX(\psi_1\opU\psi_2))\in\sub(\psi)$;
    \item if $\psi_1\opR\psi_2\in \sub(\psi)$, then $\psi_2\wedge(\psi_1\vee\opXd(\psi_1\opR\psi_2))\in\sub(\psi)$.
 \end{itemize}

	%Moreover, the set $\csub(\phi)$ stands for the set containing all closed subformulas of $\phi$.

    \subsection{Illustrating Properties Expressed by \bltl}
    \label{SubSec: Example_Props}
    In this section, we demonstrate the expressiveness of \bltl.
    Since \bltl has the ability to express Boolean logic and arithmetic operations, we can see that many concerning properties can be specified using \bltl.

    We can partition a vector into segments of varying widths, 
    and then define a Boolean function, denoted by $e_i$, to extract the $i$-th segment with width of $n$,
    namely, $e_i:\domB^m \rightarrow \domB^n$, where $m$ is the width of vector $\bm{b}$. 
    We use $\bm{b}[i]$ to refer to $e_i\left(\bm{b}\right)$
    in the case that $e_i\left(\bm{b}\right) \in \domB$.
    \paragraph{Local Robustness.}
    Given a \bnn $\network{N}$ and a $n$-width input $\bm{u}$,
    $\network{N}$ is robust w.r.t. $\bm{u}$,
    if all inputs in the region $B\left(\bm{u},\epsilon \right)$, are classified into the same class as $\bm{u}$ \cite{baluta2019quantitative}.
    Here, we consider $B\left(\bm{u}, \epsilon \right)$ as the set of vectors that differ from $\bm{u}$ in at most $\epsilon$ positions, where $\epsilon$ is the maximum number of positions at which the values differ from those of $\bm{u}$.
    The local robustness can be described as follows:
     \[
        \forall \bm{x} \in \domB^n.
        \sum_{i=1}^{|\bm{u}|} \left(\bm{x}[i] \oplus \bm{u}[i] \right) \le \epsilon
        \rightarrow
        \network{N}\left(\bm{x}\right) = \network{N} \left(\bm{u}\right)
     \]
    \paragraph{Individual Fairness.}
    In the context of a \bnn $\network{N}$ with an input of $t$ attributes and $n$-width,
    where the $s$-th attribute is considered sensitive,
    $\network{N}$ is fair w.r.t the $s$-th attribute, 
    when no two input vectors in its domain differ only in the value of the $s$-th attribute and yield different outputs \cite{zhang2020white,zheng2022neuronfair}.
    The individual fairness can be formulated as:
    \begin{align*}
        \forall \bm{a},\bm{b} \in \domB^n.
        \left(
            \neg \left(e_s(\bm{a}) = e_s(\bm{b})\right) \wedge
            \forall i \in [t] - \{s\}. e_i(\bm{a}) = e_i(\bm{b})
        \right)
        \rightarrow \network{N} \left( \bm{a} \right) = \network{N} \left( \bm{b} \right)
    \end{align*}
    where $e_i$ denotes the extraction of the $i$-th attribute, $\domB^n$ is the domain of $\network{N}$, and $\bm{a}$, $\bm{b}$ are input vectors.

    In practice, 
    it is possible to select inputs in the $\domB^n$,
    and modify the sensitive attribute to obtain the proper pairs, 
    which only differ in the sensitive attribute.
    For any such pair  $(\bm{b}, \bm{b}')$,we formulate the specification as
    $\network{N}(\bm{b}) = \network{N}(\bm{b}')$.
        
    \paragraph{{Specification for Internal Blocks.}}
    \bltl can specify block-level properties.
    For instance, the formula 
    \[
        \forall \bm{x} \in \domB^4. \opF \left(\bm{x} \ge \bm{a} \rightarrow \opN \bm{x} = \bm{a} \right)  
    \]
    states that there exists a block in the network that behaves as follows:
    for any $4$-bit input whose value is greater than or equal to $\bm{a}$, 
    the corresponding output is equal to $\bm{a}$.
	
 %%------------------------------
	\section{From \bltl to Automata}
	\label{Sec: Automata_Cons}
	
	In this section, we present both an explicit and an implicit construction that
    translate a \bltl formula into an equivalent finite-state automaton.
	We first show how to eliminate the placeholders $\opN^k$ in terms $\opN^k\bm{t}$ and atomic formulas $\bm{t}_1\sim\bm{t}_2$.

	%----------------------------
    \subsection{Eliminating Placeholders}
    \label{SubSec: Cons}
    To eliminate the placeholders $\opN^k$ in terms $\opN^k\bm{t}$,
    we define the \emph{apply operator} $[~]:
     \bm{T} \times \bm{\Sigma}\setminus \bm{\Sigma_\V} \rightarrow \bm{T}$,
    where $\bm{T}$ denotes the set of terms.
  $[\bm{t},f]$, written as $\bm{t}[f]$, is called the \emph{application} of the term $\bm{t}$ w.r.t. the Boolean function $f \in \bm{\Sigma}$, which instantiates the innermost placeholder of the term $\bm{t}$
    by the Boolean function $f$. Below, we give a formal description of the application.
    %To formalize the apply operation $\bm{T}$, we introduce the canonical form
    %of terms.
    
    Let us fix a term $\bm{t}$.
	According to Proposition \ref{Lem:OpN_Power},  $\bm{t}$ can be equivalently transformed into
	the following canonical form
	\[
	\opN^{\ell_k}g_{k-1}\left(\opN^{\ell_{k-1}}g_{k-2}\left(\cdots        g_{0}\left(\opN^{\ell_0}
	\bm{b}\right)\cdots\right)\right)
	\]
	%\fs{$\bm{M}$ is matrix of functions of the BNN. Using other notation.}
	where $\bm{b}$ is a vector constant, $\ell_0\geq0$ and  $\ell_i > 0$ for each $i>0$.
    Hereafter, we assume that $\bm{t}$ is in the canonical form, and let $\len(\bm{t})=\sum_{i=0}^k\ell_i$.
%	We in the sequel use $\len(\bm{t})$ to denote the number $\sum_{i=0}^{k}\ell_i$, called the
%	\emph{length} of $\bm{t}$. Obviously, if $\len(\bm{t})=0$, $\bm{t}$ is the vector constant $\bm{b}$.

    When $\bm{t}$ is $\rhd$-free, i.e., $\len(\bm{t})=0$, we let $\bm{t}[f]=\bm{t}$. 
    When $\len(\bm{t})>0$,  we say that the Boolean function $f\in\bm{\Sigma}$ is \emph{applicable}
	w.r.t. the term $\bm{t}$, if:
	\begin{enumerate}
		\item $\bm{b}\in\dom {f}$;
		\item if $\ell_0=1$, then  $\ran{f}=\dom{g_0}$. 
	\end{enumerate}
     Intuitively, the above two conditions ensure that 
     $f\left(\bm{b}\right)$ and $g_{0}\circ f$ are well-defined.
 %   indicate that $f$'s domain and range are fit for the adjacent vectors and functions.
   
   If  $f\in\bm{\Sigma}$ is applicable
	w.r.t. the term $\bm{t}$, we let $\bm{t}[f]$ be the term:
	\[
	\bm{t}[f]=\begin{cases}
		\opN^{\ell_k}g_{k-1}\left(\opN^{\ell_{k-1}}g_{k-2}\left(\cdots g_{0}\left(\opN^{\ell_0-1}
		\bm{b}'\right)\cdots\right)\right), & \mbox{if } \ell_0>1 \\
		\opN^{\ell_k}g_{k-1}\left(\opN^{\ell_{k-1}}g_{k-2}\left(\cdots g_{1}\left(\opN^{\ell_{1}}
		\bm{b}''\right)\cdots\right)\right), & \mbox{if } \ell_0=1
	\end{cases}
	\]
	where $\bm{b}'= f\left(\bm{b}\right)$ and $\bm{b}''=\left(g_{0}\circ f \right) \left(\bm{b}\right)$.
   
    It can be seen that $\len(\bm{t}[f])=\len(\bm{t})-1$. By iteratively applying this operator,
    the placeholders $\opN^k$ in the term $\bm{t}$ can be eliminated.
%	Hence, if $\len(\bm{t})>0$, then $\len(\bm{t}[f])=\len(\bm{t})-1$.
%	If $\len(\bm{t})=0$ (i.e., $\bm{t}$ is $\rhd$-free), each $f$ is applicable w.r.t. it, and in this case we let $\bm{t}[f]=\bm{t}$.
	%
	%
	For convenience, we write $\bm{t}[f_0,f_1,\ldots,f_i]$ for the shorthand of
	$$\bm{t}[f_0][f_1]\cdots[f_i],$$
	provided that each Boolean function $f_i$ is applicable w.r.t. $\bm{t}[f_0][f_1]\cdots[f_{i}]$.
	Likewise, we call $\bm{t}[f_0,f_1,\ldots,f_i]$ the \emph{application} of $\bm{t}$ w.r.t.
the Boolean functions $f_0,f_1,\cdots, f_{i}$.
	
	%We say that a set of closed terms $\{\bm{t}_1,\ldots,\bm{t}_n\}$ are \emph{compatiable} if there exists
	%some $\bm{M}\in\bm{\Sigma}$ which is appicatable w.r.t. each $\bm{t}_i$.
    
   In particular, the \emph{collapsion} of term $\bm{t}$, denoted by $\bm{t}\downarrow$, 
   is the term  
  %$$\left(g_k\circ g_{k-1}\circ\cdots\circ g_{0}\right)\left(\bm{b}\right),$$
   %that is the application 
   $\bm{t}[\underbrace{\bm{I},\ldots,\bm{I}}_{\len(\bm{t})}]$, namely, $\bm{t}\downarrow$ is
   obtained from $\bm{t}$ w.r.t. $\len(\bm{t})$ identity functions.

%	then the \emph{collapsion} of $\bm{t}$ (denoted as $\bm{t}\downarrow$) is the binary vector
%	$$\left(g_k\circ g_{k-1}\circ\cdots\circ g_{0}\right)\left(\bm{b}\right).$$
%	Indeed, $\bm{t}\downarrow$ is just $\bm{t}[\underbrace{\bm{I},\ldots,\bm{I}}_{\len{\bm(t)}}]$.

We hereafter denote by $\cons(\bm{\Sigma})$ the set of constraints $\bm{t}_1\sim\bm{t}_2$ over the signature $\bm{\Sigma}$
and lift the apply operator $[~]$ from terms to atomic formulas $\bm{t}_1\sim\bm{t}_2$.
%
%	\emph{Constraints} over $\bm{\Sigma}$ are atomic formulas, and each of them is in the form of
%	$\bm{t}_1\sim\bm{t}_2$ where each $\bm{t}_i$ is a term.
%	We in what follows denote by $\cons(\bm{\Sigma})$ the set consisting of constraints over $\bm{\Sigma}$.
	%
	For a constraint $\gamma=\bm{t}_1\sim\bm{t}_2\in\cons(\bm{\Sigma})$,
    we denote by $\gamma[f]$ the constraint $\bm{t}_1[f]\sim\bm{t}_2[f]$;
    and by ${\gamma\downarrow}$ the constraint ${\bm{t}_1\downarrow}\sim{\bm{t}_2\downarrow}$.
	Note that the former implicitly assumes that the Boolean function $f$ is applicable w.r.t.
	both terms $\bm{t}_1$ and $\bm{t}_2$ (in this case, we call that $f$ is applicable w.r.t. $\gamma$),  whereas
	the latter requires that the terms ${\bm{t}_1\downarrow}$ and ${\bm{t}_2\downarrow}$ have the same width
    (we call that $\bm{t}_1$
	and $\bm{t_2}$ are \emph{compatible} w.r.t.  collapsion).
	In addition, we let $\len(\gamma)=\max(\len(\bm{t}_1),\len(\bm{t}_2))$, and in the case
	that $\len(\gamma)=0$, we let $\gamma[f]=\top$ (resp. $\gamma[f]=\bot$) for
	any Boolean function $f$ if $\gamma$ is evaluated to true (resp. false).
	
	We subsequently extend the above notations to constraint sets. Suppose that $\Gamma\subseteq\cons(\bm{\Sigma})$,
	we let $\Gamma[f]\defeq\{\gamma[f]\mid \gamma\in\Gamma\}$, and let ${\Gamma\downarrow}\defeq\{{\gamma\downarrow}\mid\gamma\in\Gamma\}$. Remind that the notation $\Gamma[f]$ makes sense only
	if the Boolean function $f$ is \emph{applicable} w.r.t. $\Gamma$, namely $f$ is applicable w.r.t. each constraint $\gamma\in\Gamma$.
	Likewise, the notation ${\Gamma\downarrow}$ indicates that  $\bm{t}_1$ and $\bm{t}_2$ is compatible w.r.t. collapsion
	for each constraint $\bm{t}_1\sim\bm{t}_2\in \Gamma$.

	\begin{theorem}
		\label{Thm: Intuition}
		For a \bnn\ $\network{N}$ given by
  $f_\network{N} = f_{n-1}\circ f_{n-2}\circ\cdots\circ f_1\circ f_0$, %which is represented via the boolean function sequence
  %matrix sequence
	%	$$\bm{N}_0,\bm{N}_1,\cdots,\bm{N}_{n-1}$$
 and a constraint $\gamma\in\cons(\bm{\Sigma})$, %satisfying for $\network{N}$,
 we have:
		\begin{enumerate}
			\item $\network{N}, i\models \gamma$ iff $\network{N}, i+1\models\gamma[f_i]$ for each $i<n$.
			\item $\network{N}, i\models\gamma$ iff $\network{N}, i\models\gamma\downarrow$ for each $i\geq n$.
		\end{enumerate}
	\end{theorem}
	Indeed, since $\gamma\downarrow$ must have the form $\bm{b}_1\sim\bm{b}_2$, where both $\bm{b}_1$
	and $\bm{b}_2$ are Boolean constants, then the truth value of $\gamma\downarrow$ can always be directly evaluated.
	
	%-------------------------------
	\subsection{Automata Construction}
	\label{SubSec: Automata_Cons}

	Given a \bltl formula $\varphi$ in \nnf,
	we can construct a finite-state automaton  $\automaton{A}_\varphi=(Q_\varphi,\bm{\Sigma}, \delta_\varphi, I_\varphi, F_\varphi)$,
	where:
%$Q_\varphi = 2^{\sub(\varphi)}$

 \begin{itemize}
    \item $Q_\varphi = \bigcup_{\Gamma\subseteq\sub(\varphi)}\cl(\Gamma)$.
		Recall that $\cl(\Gamma)\subseteq 2^{\sub(\varphi)}$ if $\Gamma\subseteq\sub(\varphi)$, thus each state must be a subset of $\sub(\varphi)$.
  \item For each $q\in Q_\varphi$, let $\cons(q)\defeq q\cap\cons(\bm{\Sigma})$, let $q'=\{\psi\mid\opX\psi\in q\}$
		and let $q''=\{\psi\mid\opXd\psi\in q\}$. Then, for each Boolean function $ f \in\bm{\Sigma}$, we have
		$$\delta_\varphi(q, f)=\begin{cases}
			\emptyset, & \bot\in q \\
			\cl(q'\cup q''\cup\cons(q)[f]),
			& \bot\not\in q
		\end{cases}.$$
		\item $I_\varphi=\{q\in Q_\varphi\mid \varphi\in q\}$ is the set of initial states.
		\item $F_\varphi$ is the set of accepting states such that for every state $q\in Q_\varphi$, $q\in F_\varphi$ only if 
  $\{\psi\mid\opX\psi\in q\}=\emptyset$, $\bot\not\in q$ and $\cons(q)\downarrow$ is
		evaluated true.
	\end{itemize}

        For a \bnn $\network{N}$ given by $f_\network{N} = f_{n-1}\circ f_{n-2}\circ\cdots\circ f_1\circ f_0$,
        we denote by $\network{N}\in\lang{\automaton{A}_\varphi}$ if the sequence of the Boolean functions
        $f_0,f_1,\cdots, f_{n-1}$, regarded as a finite word, is accepted by the automaton $\automaton{A}_\varphi$.
        
        Intuitively, 
        $\network{N}$ accepts an input word 
        iff it has an accepting run $q_0, q_1 \cdots, q_n$, where $q_i$ is constituted with a set of formulas that make the specification $\varphi$ valid at the position $i$.
        In this situation, $I_\varphi$ refers to the states involving $\varphi$ and $q_0 \in I_\varphi$.
        %to ensure that $\varphi$ is satisfied at position $0$.
        For the transition $q_i \xrightarrow{f_i} q_{i+1}$, 
        $q_i'$ and  $q_i''$ indicate the sets of formulas which 
        should be satisfied in the next position $i+1$ 
        according to the semantics of \emph{next} ($\opX$) and \emph{weak next} ($\opXd$).
        Additionally, 
        $\cons(q_{i+1})$ is obtained by applying the Boolean function $f_i$ to the constraints in $q_i$.

	The following theorem reveals the relationship between $\varphi$ and $\automaton{A}_\varphi$.

	\begin{theorem}
		\label{Thm: Automata-Cons}
		%Given a closed \bltl formula $\varphi$, there exists an automaton $\automaton{A}_\varphi$  such that
            Let  $\network{N}$ be a \bnn given by a sequence of Boolean functions 
            for a \bltl formula $\varphi$,
            we have:
		\begin{center}
  $\network{N}\models\varphi$ if and only if $\network{N}\in\lang{\automaton{A}_\varphi}$. %for each
	%	\bnn\ $\network{N}$. % (when it is encoded into a sequence over $\bm{\Sigma}$).
  \end{center}
	\end{theorem}
	The proof is given in Appendix \ref{Appendix: Automata-Cons} and an example of the construction is given in Appendix~\ref{Appendix:example}.

 	\subsection{Tableau-Based Construction}
	\label{SubSec: Tableau}
        \iffalse
        In this section, we present an approach for incrementally constructing the automaton 
        via tableau rewriting. 
        Intuitively, we just construct one part of the automaton at a time, and then
        check its associated formulas to determine 
        whether to continue to construct other parts or stop
        based on the check result.
        \fi

        We have successfully provided a process for converting an \bltl formula into an automaton on finite words. At first glance, it seems that the model checking problem w.r.t. \bnn can
        be immediately boiled down to a word-problem of finite automata. 
        Nevertheless, a careful analysis shows that this would result in a prohibitively high cost.
	   Actually, for a \bltl formula $\varphi$, 
        the state set of $\automaton{A}_\varphi$
	is $\bigcup_{\Gamma\subseteq\sub(\varphi)}\cl(\Gamma)\subseteq 2^{\sub(\varphi)}$,
	thus the number of states is exponential in the size of the length of $\varphi$.
       To avoid explicit construction, we provide an ``on-the-fly'' approach when performing synthesis.

        Suppose the \bltl $\varphi$ is given in \nnf and the \bnn 
        $\network{N}$ is given as a sequence of Boolean functions $f_0,f_1,\ldots,f_{n-1}$, using the following
	approach, we may construct a tree $\tableau{T}_{\varphi,\network{N}}$ which fulfills the followings:
        \begin{itemize}
		\item $\tableau{T}_{\varphi,\network{N}}$ is rooted at $\langle 0, \{\varphi\}\rangle$;
		\item For an internal node $\langle i, \Gamma\rangle$ with $i<n-1$, it has a child
		$\langle j, \Gamma'\rangle$ only if there is a tableau rule
		\[
		\trule{i}{\quad \Gamma\quad }{j}{\quad \Gamma'\quad}
		\]
		where $j$ is either $i$ or $i+1$.
		\item 		A leaf $\langle i, \Gamma\rangle$ of $\tableau{T}_{\varphi, \network{N}}$ is a \tname{Modal}-node with $i=n-1$, where 
  nodes to which only the rule \tname{Modal} can be applied are called \tname{Modal}-nodes.
	\end{itemize}

	\begin{figure}[ht]
		\hrule
		
		\vskip 2mm
		\[
		\tname{And} \quad  \trule{i}{\Gamma,\varphi_1\wedge\varphi_2}{i}{\Gamma,\varphi_1,\varphi_2}
		\qquad\qquad
		\tname{Or-$j$} \quad  \trule{i}{\Gamma,\varphi_1\vee\varphi_2}{i}{\Gamma,\varphi_j}
		\quad (j=1,2)
		\]
		\[
		\tname{True}\quad
		\trule{i}{\Gamma,\bm{t}_1\sim\bm{t}_2}{i}{\Gamma,\top}
		\qquad\qquad
		\tname{False}\quad
		\trule{i}{\Gamma,\bm{t}_1\sim\bm{t}_2}{i}{\Gamma,\bot}
		\]
		\[
		\tname{Until}\quad \trule{i}{\Gamma,\varphi_1\opU\varphi_2}{i}{\Gamma,\varphi_2\vee(\varphi_1\wedge\opX(\varphi_1\opU\varphi_2))}
		\]
		\[
		\tname{Release} \quad
		\trule{i}{\Gamma,\varphi_1\opR\varphi_2}{i}{\Gamma,\varphi_2\wedge(\varphi_1\vee\opXd(\varphi_1\opR\varphi_2))}
		\]
		\[
		\tname{Modal}\quad
		\trule{i}{\Gamma,\opX\psi_1,\ldots,\opX\psi_m,\opXd\varphi_1,\ldots,\opXd\varphi_k}{i+1}
		{\{\gamma[f_i]\mid \gamma\in\Gamma,\len(\gamma)>0\},\psi_1,\ldots,\psi_m,\varphi_1,\ldots,\varphi_k}
		\]
		
		\hrule
		\vskip 2mm
		\caption{Tableau rules for Automata Construction}
		\label{Fig: Tab-rules}
	\end{figure}
	%-------------------------------
        Tableau rules are listed in Figure \ref{Fig: Tab-rules}. For the rule \tname{Modal}, we require that
        $\Gamma$ consists of atomic formulas being of the form $\bm{t}_1\sim\bm{t}_2$.
        In the rules \tname{True} and \tname{False}, we require that $\len(\bm{t}_1\sim\bm{t}_2)=0$
        and it is evaluated to true and false, respectively.

        Suppose $\langle n, \Gamma\cup\{\opX\psi_1,\ldots,\opX\psi_m\}\cup\{\opXd\varphi_1,\ldots,\opXd\varphi_k\}\rangle$ is a leaf of $\tableau{T}_{\varphi,\network{N}}$. We say it is \emph{successful} if
	$m=0$ and $\Gamma\downarrow$ is evaluated to true.
	In addition, we say a path of $\tableau{T}_{\varphi,\network{N}}$ is \emph{successful} if
    it ends with a successful leaf, and no node along this path contains $\bot$.
    
        In the process of the on-the-fly construction,
        we start by creating the root node, 
        then apply the tableau rules to rewrite the formulas in the subsequent nodes.
        In addition,  
        before the rule \tname{Modal} or \tname{Or-$j$} is applied, 
        we preserve the set of formulas,
        which allows us to trace back and construct other parts of the automaton afterward.
        We exemplify how
            to achieve the synthesis task via the construction in Section \ref{Sec: Synthesis}.
        
	\begin{theorem}
		\label{Thm: Model_Checking}
		$\network{N}\models\varphi$ if and only if $\tableau{T}_{\varphi,\network{N}}$ has a successful path.
	\end{theorem}
%	\todo{Add the size of $\tableau{T}_{\varphi,\network{N}}$ and complexity of the model checking}

    \begin{proof}
	Let $\automaton{A}_\varphi$ be the automaton  corresponding to $\varphi$. According to Theorem \ref{Thm: Automata-Cons}, it suffices to show that $\network{N}\in\lang{\automaton{A}_\varphi}$ iff  $\tableau{T}_{\varphi,\network{N}}$ has a successful path.
     \par
     Suppose, $\network{N}$ is accepted by $\automaton{A}_\varphi$ with the run $q_0,q_1,\ldots, q_n$, we also create the root node $\langle 0, \Gamma_0=\{\varphi\}\rangle$. Inductively, we have the followings statements for each node $\langle i,\Gamma_j\rangle$ which is already constructed:
     \begin{enumerate}[1)]
         \item \label{Item:Cons_1} $\Gamma_j\subseteq q_i$;
         \item \label{Item:Cons_2} $\network{N}, i\models\psi$ for each $\psi\in q_i$ (see the proof of Thm. \ref{Thm: Automata-Cons})
     \end{enumerate}
     Then, if $\langle i,\Gamma_j\rangle$ is not a leaf, we create a new node $\langle i',\Gamma'_j\rangle$ in the following way:
     \begin{itemize}
         \item $i'=i$ if $\langle i,\Gamma_j\rangle$ is not a \tname{Modal}-node, otherwise $i'=i+1$;
         \item if rule \tname{Or-$k$} ($k=1,2$) is applied to $\langle i,\Gamma_j\rangle$ to some $\varphi_1\vee\varphi_2\in\Gamma_j$, we require that $\varphi_k\in\Gamma_j$;
         for other cases, $\Gamma'_j$ is uniquely determined by $\Gamma_j$ and the tableau rule which is applied.
     \end{itemize}
     It can be checked that both Items \ref{Item:Cons_1}) and \ref{Item:Cons_2}) still hold at $\langle i',\Gamma'_j\rangle$. Then, we can see that the path we constructed is successful since $q_n$ is an accepting state of $\automaton{A}_\varphi$.

     \par For the other way round, suppose that $\tableau{T}_{\varphi,\network{N}}$ involves a successful path
     \[
     \begin{array}{ll}
     \langle 0, \Gamma_{0,0}\rangle, \langle 0,\Gamma_{0,1}\rangle,\ldots,\langle 0, \Gamma_{0,\ell_0}\rangle, \langle 1, \Gamma_{1,0}\rangle, \langle 1, \Gamma_{1,1}\rangle, \ldots,
     \langle 1, \Gamma_{1,\ell_1}\rangle, \ldots,\\
     \langle i, \Gamma_{i,0}\rangle, \langle i, \Gamma_{i,1}\rangle, \ldots,\langle i, \Gamma_{i,\ell_i}\rangle,\ldots,
     \langle n, \Gamma_{n,0}\rangle, \langle n, \Gamma_{n,1}\rangle, \ldots,\langle n, \Gamma_{n,\ell_n}\rangle
     \end{array}\]
     then, the state sequence $q_0, q_1, \ldots, \ldots, q_n$ yields an accepting run of $\automaton{A}_\varphi$ on $\network{N}$, where $q_i=\bigcup_{j=0}^{\ell_i}\Gamma_{i,j}$.
     \qed
     \end{proof}   
        
	\section{\bnn Synthesis}
	\label{Sec: Synthesis}

    \iffalse   
    \begin{figure}[t]
    \centering
    \includegraphics[width=0.9\textwidth]{overview.png}
    \caption{Overview of synthesis framework.}
    \label{Fig.overview}
    \end{figure}
    \fi
    
    Let us now consider a more challenging task: Given a \bltl specification $\varphi$,  to find some \bnn $\network{N}$ such that $\network{N}\models\varphi$.
    In the synthesis task,
    the parameters of the desired \bnn are not given,
    even, we are not aware of the length (i.e., the number of blocks) of the network.
    To address this challenge, we leverage the tableau-based method (cf. Section \ref{SubSec: Tableau}) to construct the automaton for the given specification $\varphi$ and check the existence of the desired \bnn at the same time. But when performing
    the tableau-based rewriting, we need to view each block (i.e., a Boolean function) $f_i$ as an unknown variable (called \emph{block variable} in what follows).

    The construction of the tableau-tree starts from the root node $\langle 0, {\varphi} \rangle$.
    During the construction, for each internal node $\langle i, \Gamma \rangle$, the following steps are followed:
    Initially, rules other than \tname{Or-$1$} and \tname{Modal} are applied to $\Gamma$ until no further changes occur.
    Then rule \tname{Or-$j$} is applied to the disjunctions in the formula set,
    and we always first try rule \tname{Or-$1$} when the rewriting is performed. Lastly, rule \tname{Modal} is applied to generate node $\langle i+1, \Gamma' \rangle$, which becomes the next node in the path, and the Boolean function $f_i$ used in the rewriting is just a {block variable}.
    Particularly, we
    retain a stack of formula sets on each of those either \tname{Or-$j$} or \tname{Modal} is applied for tracing back.
    Once an $\opX$-free \tname{Modal}-node is reached,
    We verify the success of the path.
    However,
    since now the blocks are no longer concrete in this setting, 
    an atomic formula of the form $\gamma[f_i,\ldots, f_{i+k}]$ cannot be immediately evaluated even if it is $\rhd$-free.
    As a result, whether a path is \emph{successful} cannot be evaluated directly.
            
    To settle this, we invoke an integer different logic (\idl)
	solver to examine the satisfiability of the atomic formulas in the \tname{Modal}-nodes along the path,
	and we declare success if all of them are satisfiable and it in addition ends up with an
	$\opX$-free \tname{Modal}-node.
        Meanwhile, the model given by the solver would reveal hyper-parameters of the \bnn,
        which then we adopt to obtain the expected \bnn.
        For a node $\langle i,\Gamma\rangle$, we call
	$i$ to be the \emph{depth counter}.
	Once the infeasibility is reported by
	the \idl solver, or some specific depth counter (call it the \emph{threshold}) is reached, a trace-back to the 
        nearest \tname{Or-$1$} node is required:
	nodes under that node are removed, and then use \tname{Or-$2$} for that, but this time
	we do not push anything into the stack, because both choices for the disjunctive formula have
	been tried so far. If no \tname{Or-$1$} nodes remains in the stack when doing trace-back,
	we declare the failure of the synthesis.

        Now, there are two issues to deal with during that process. The first is, how to determine if
	the aforementioned `threshold' is reached; second, how can we convert the satisfiability testing
	into \idl-solving.

    \subsection{The Threshold}
    \label{SubSec: Threshold}
	There exists a na{\"i}ve bound for the first problem, which is just the state number of
	$\automaton{A}_\varphi$.
	However, this bound is in general not compact (i.e., doubly exponential in the size of the formula $\varphi$), and thus we provide another tighter bound.
 
	We first define the following notion:
	We call two modal nodes  $\langle i, \Gamma\rangle$ and $\langle j, \Gamma'\rangle$
	are \emph{isomorphic}, denoted as $\langle i,\Gamma\rangle\cong\langle j,\Gamma'\rangle$, if $\Gamma$ can be transformed into $\Gamma'$ under a (block) variable bijection.
	The following lemma about isomorphic model nodes is straightforward.

	\begin{lemma}
		\label{Lem: Cong}
		If $\langle i,\Gamma\rangle\cong\langle j,\Gamma'\rangle$ and
		$\langle i,\Gamma\rangle$ could lead to a successful leaf (i.e., satisfiable leaf), then so does $\langle j,\Gamma'\rangle$.
	\end{lemma}

	Thus, given $\varphi$, the threshold can be the number of equivalence classes w.r.t. $\cong$.
	To make the analysis clearer, we here introduce some auxiliary notions.
	\begin{itemize}
		\item We call an atomic constraint $\gamma$ occurring in $\varphi$ to be
		an \emph{original constraint} (or, \emph{non-padded constraint}); and
		call a formula being of the form $\gamma[f_i,\ldots,f_j]$
		\emph{padded constraint}, where $f_i,\ldots,f_j$ are block variables.
		\item A (padded or non-padded) constraint with length $0$ (i.e., $\rhd$-free) is called
		\emph{saturated}. In general, such a constraint is obtained
		from a non-padded constraint $\gamma$ via applying $k$ layer variables,
		where $k=\len(\gamma)$.
	\end{itemize}

	\begin{theorem}
		\label{Thm: Threshold}
		Let $\varphi$ be a closed \bltl formula, and let
		\begin{itemize}
			\item $c=\#(\cons(\bm{\Sigma})\cap\sub(\varphi))$, i.e., the number
			of (non-padded) constraints occurring in $\varphi$;
			\item $k=\max\{\len(\gamma)\mid\gamma\in\cons(\bm{\Sigma})\cap\sub(\varphi)\}$, i.e., the maximum length of non-padded constraints occurring in $\varphi$;
			\item $p$ be the number of temporal operators in $\varphi$
		\end{itemize}
		then,  $2^{(k+1)c+p}+1$ is a threshold for synthesis.
	\end{theorem}

	The proof is shown in Appendix \ref{Appendix: Threshold}.
    \subsection{Encoding with \idl Problem}
    \label{SubSec: Encoding}
    Another problem is how to convert the satisfiability testing 
    into \smt-solving. 
    To tackle this, 
    we present a method that transforms \bltl atomic formulas to \idl constraints.

    We may temporarily view a Boolean function $g:\domB^m\to\domB^n$ as a (partial) integer function with domain $[2^m]$,
    namely, we equivalently view $g$ maps $\dec(\bm{b})$ to $\dec(g(\bm{b}))$.
    
    For a $\rhd$-free term $\bm{t}=(f_k\circ f_{k-1}\circ\cdots\circ f_0)(\bm{b})$, we say that 
    $(f_i\circ f_{i-1}\circ\cdots\circ f_0)(\bm{b})$ is an \emph{intermediate term} of $\bm{t}$ where $i\leq k$.
    In what follows, we denote by $\bm{T}$ the set of all intermediate terms that may occur in the
    process of \smt-solving, 
    which is a part of synthesis that check the satisfiability of atomic formulas in successful leaves.

    Remind that in a term or an intermediate term, a symbol $g$ may either be a fixed function or a 
    variable that need to be determined by the \smt-solver (i.e., \emph{block variables}). To make it clearer, we in general use
    $g_0, g_1, \ldots$ to designate the former functions, whereas use $f_0, f_1$, etc for the latter
    cases. 

    The theory of \idl is limited to handling the difference constraints of the form $x - y \sim c$, 
    where $x$, $y$ are integer variables and $c$ is an integer constant.
    However, since functions occur in the terms, they cannot be expressed using \idl. 
    To this end, 
    we note that we merely care about partial input-output relations of the functions,
    which consist of mappings among $\bm{T}$, 
    and then the finite mappings can be expressed by integer constraints.
    Thus, for each intermediate term $\bm{t}\in\bm{T}$,  we introduce an integer variable $v_{\bm{t}}$.

    Then, all constraints describing the synthesis task are listed as follows.
    \begin{enumerate}[(1)]
        \item For each \bltl constraints $\bm{t}_1\sim\bm{t}_2$, we have a conjunct $v_{\bm{t}_1}\sim v_{\bm{t}_2}$.
        \item For each block variable $f:\domB^n\to\domB^m$ and each $f(\bm{t})\in\bm{T}$, we add the bound constraints
        $0\leq v_{f(\bm{t})}$ and $v_{f(\bm{t})}\leq 2^m$. 
        \item For each block variable $f$ and every $\bm{t}_1,\bm{t}_2\in\bm{T}$, we have
        $v_{\bm{t}_1}=v_{\bm{t}_2}\rightarrow v_{f(\bm{v}_1)} = v_{f(\bm{v}_2)}$, which guarantees $f$ to be a mapping.
        \item For every fixed function $g$, we impose the constraint 
        $v_{g(\bm{t})}=\dec(g(\bin(v_{\bm{t}})))$ for every $\bm{t}\in\bm{T}$.
    \end{enumerate}

    Once the satisfiability is reported by the \smt-solver, 
    we extract partial mapping information of $f_i$'s from the solver's model,
    by analyzing equations of the form $v_{\bm{t}} = c$, where $c$ is an integer called the value of $\bm{t}$.  
    We iterate over the model and record the value of terms, 
    when we encounter an equation in the form of $v_{f_i(\bm{t})} = c$, we query the value of $\bm{t}$,
    and obtain one input-output relation of $f_i$.
    \iffalse
    For example, if we encounter the equations $v_{f_1\left(\bm{t}\right)} = 9$ and $v_{\bm{t}} = 3$, 
    we know that $f_1(3) = 9$, 
    and we also learn that
    the \bnn has at least two blocks (recall that $f_1$ represent the second block of the \bnn).
    \fi
    Eventually, we get partial essential mapping information of suc $f_i$'s.

    \subsection{Utilize the Synthesis}
    \label{SubSec: Finally_Synthesis}
    A \bnn that satisfies the specification can be obtained via block-wise training, 
    namely, training each block independently to fulfill its generated input-output mapping relations, which is extracted by the \smt-solver during the synthesizing process.
    Indeed, such training is not only in general lightweight, but also able to reuse the 
    pre-trained blocks.
 
    Let us now consider a more general requirement that we have both high-level temporal specification (such as fairness, robustness) and data constraints (i.e., labels on a dataset), and is asked to obtain a \bnn to meet all these obligations.

    A straightforward idea is to express all data constraints with \bltl, and then perform a monolithic synthesis. However, such a solution seems to be infeasible, because the large amount of data
    constraints usually produces a rather complicated formula, and it makes the synthesis extremely difficult.

    An alternative approach is to first perform the synthesis w.r.t. the high-level specification, then do a retraining upon the dataset.  However, the second phase may distort the result of the first phase. In general, one need to conduct an iterative cycle composed of synthesis-training-verification, yet the convergence of such process cannot be guaranteed. Thus, we need make a trade-off between these two types of specifications.

    More practically, synthesis is used as an ``enhancement'' procedure. Suppose, we already have some \bnn trained with the given dataset, then we are aware the hyper-parameters of that. 
    This time, we have more information when doing synthesis, e.g., the threshold is replaced by the length of the network, and the shape (i.e., the width of input and output) of each block are also given. With this, we may perform a more effective \smt-solving process, and then retrain each block individually. Definitely, this might affect the accuracy of network, and some compromise also should be done.

 %----------------------------------------------------
    \section{Experimental Evaluation}
    \label{Sec:Expr_Eval}
    We implement a prototype tool in Python,
    which uses Z3 \cite{de2008z3} as the off-the-shelf \idl solver and PyTorch to train blocks and {\bnn}s.    
    To the best of our knowledge, few existing work on synthesizing \bnn has been done so far.
    Hence, we mainly investigate the feasibility of our approach by exploring how much the trustworthiness of \bnn can be enhanced, 
    and the corresponding trade-off on accuracy degradation.
    The first two experiments focus on evaluating the effectiveness of synthesis in enhancing the properties of {\bnn}s
    We set {\bnn}s with diverse architectures as baselines, and synthesize models via the ''enhancement'' procedure, wherein the threshold matches the length of the baselines, and the shape of blocks are constrained to maintain the same architecture as the baselines.
    Eventually, the blocks are retrained to fulfill the partial mapping, and the synthesized model is obtained through retraining on the dataset.
    We compare the synthesized models and their baselines on two properties: \emph{local robustness} and \emph{individual fairness}.
    
    Moreover, we study the potential of our approach to 
    assist in determining the network architecture.
    \paragraph{Datasets.}
    We train models and evaluate our approach over two classical datasets, \mnist \cite{deng2012mnist} and \uci \cite{uci}.

    \mnist is a dataset of handwritten digits,
    which contains 70,000 gray-scale images with 10 classes, and each image has $28 \times 28$ pixels.
    In the experiments, 
        we downscale the images to $10 \times 10$, 
        and binarize the normalized images, 
        and then transform them into $100$-width vectors.

    \uci contains 48,842 entries with 14 attributes, such as age, gender, workclass and occupation.
    The classification task on the dataset 
        is to predict whether an individual's annual salary is greater than 50K.
    We first remove unusable data, retain 45,221 entries, 
    and then transform the real-value data into 66-dimension binarized vectors as input.

    \paragraph{Experimental Setup.}
    In the block-wise training,
    different loss functions are employed for internal and output blocks: the \mse loss function for internal blocks and the cross-entropy loss function for output blocks.
    The training process entails a fixed number of epochs, with 150 epochs for internal blocks and 30 epochs for output blocks.
    The experiments are conducted on a 3.6G HZ CPU with 12 cores and 32GB RAM, and the blocks and {\bnn}s are trained using a single GeForce RTX 3070 Ti GPU.
    
    \begin{table}[htbp]
        \centering
        \caption{\bnn baselines.}
         \setlength{\tabcolsep}{3mm}{
        \begin{tabular}{ccc|ccc}
        \hline
        \textbf{Name}& \textbf{Arch}& \textbf{\acc}  & \textbf{Name}  & \textbf{Arch}  & \textbf{\acc} \\ \hline
        \rowcolor[HTML]{EFEFEF}
        {\color[HTML]{000000} \textbf{R1}} & {\color[HTML]{000000} 100-32-10}    & {\color[HTML]{000000} 82.62\%} & {\color[HTML]{000000} \textbf{F1}} & {\color[HTML]{000000} 66-32-2}    & {\color[HTML]{000000} 80.12\%} \\
        \textbf{R2}  & 100-50-10   & 84.28\%   & \textbf{F2}   & 66-20-2    & 79.88\%  \\
        \rowcolor[HTML]{EFEFEF}
        {\color[HTML]{000000} \textbf{R3}} & {\color[HTML]{000000} 100-50-32-10} & {\color[HTML]{000000}83.50\%} & {\color[HTML]{000000} \textbf{F3}} & {\color[HTML]{000000} 66-32-20-2} & {\color[HTML]{000000} 78.13\%} \\ \hline
        \end{tabular}}
        \label{Table.BNN}
        \vspace{-2.0em}
    \end{table}

    \paragraph{Baseline.}
    We use six neural networks with different architectures as baselines, where three models
    \textbf{R1}-\textbf{R3} are trained on the \mnist for 10 epochs with a learning rate of $10^{-4}$ to study on local robustness.
    For individual fairness, 
        we train 3 models (\textbf{F1}-\textbf{F3}) on the \uci for 10 epochs, with a learning rate of $10^{-3}$,
        and split the dataset into
        a training set and a test set in a 4:1 ratio.

    The detailed information is listed in Table \ref{Table.BNN},
    Column (Name) indicates the name of {\bnn}s, and 
    Column (Arch) presents their architectures.
    The architecture of each network is described as by a sequence $\{n_i\}_{i=0}^s$,
    where $s$ is the number of the blocks in the network, and
    $n_i$  and $n_{i+1}$ indicate the input and output dimensions of the $i$-th block.  For instance, 100-32-10 indicates that the \bnn has two blocks, the input dimensions of these blocks are 100 and 32 respectively, and the number of classification labels is 10.
    Column (\acc) shows the accuracy of the models on the test set.
    
    \subsection{Local Robustness}
    \label{SubSec: Robust}
    In this section, we evaluate the effectiveness of our approach for enhancing the robustness of models in different cases.
    We use the metric, called Adversarial Attack Success Rate (\as), to measure a model's resistance to adversarial attacks. 
    \as is calculated as the proportion of 
    perturbed inputs that leads to a different prediction result compared to the original input.

    We choose 30 image vectors from the training set, and 
    set the maximum perturbation to four levels,
    $\epsilon \in \{1, 2, 3, 4\}$.
    The value of $\epsilon$ indicates the number of positions that can be modified in one image vector.
    One selected input vector, one maximum perturbation $\epsilon$ and one baseline model 
        constitute a case, resulting in a total of 360 cases.

    For each of the 360 case, 
        we make a synthesized model individually, 
        and compare its \as with the corresponding baseline.
    For the local robustness property (cf. Section \ref{SubSec: Example_Props}),
    since the input space is too large to enumerate,
    we need to sample inputs within $B\left(\bm{u},\epsilon\right)$ when describing the specification,
    which is formulated as
    $\bigwedge_{i=1}^k \left(\network{N}(\bm{u}) = \network{N}(\bm{b}_i)\right)$,
        where each $\bm{b}_i$ is a sample and $k$ is the number of samples.
    We here sample $100$ points within the maximum perturbation limit $\epsilon$.
    The specification is written as 
    $\bigwedge_{i=1}^k \left(\opN^n\bm{u} = \opN^n\bm{b}_i\right)$, where $n$ is the number of the block of the baseline. 
    Subsequently, we use the block constraint (cf. Section \ref{SubSec: Encoding}), $0\leq v_{f_i(\bm{t})}\leq 2^m$, to specify the range of output of each block. 
    To make the bound tighter, we retain the maximal and minimal activations of each block using calibration data run on the baseline,
    and then take the recorded values as bounds.
    Eventually, the generated mappings are used in the block-wise training, and then the enhanced \bnn is obtained  through retraining on the \mnist dataset.
    
    \begin{figure}[t]
        \centering
        \subfigure[Arch:100-32-10]{
        \label{Fig.sub.1}
        \includegraphics[scale=0.15]{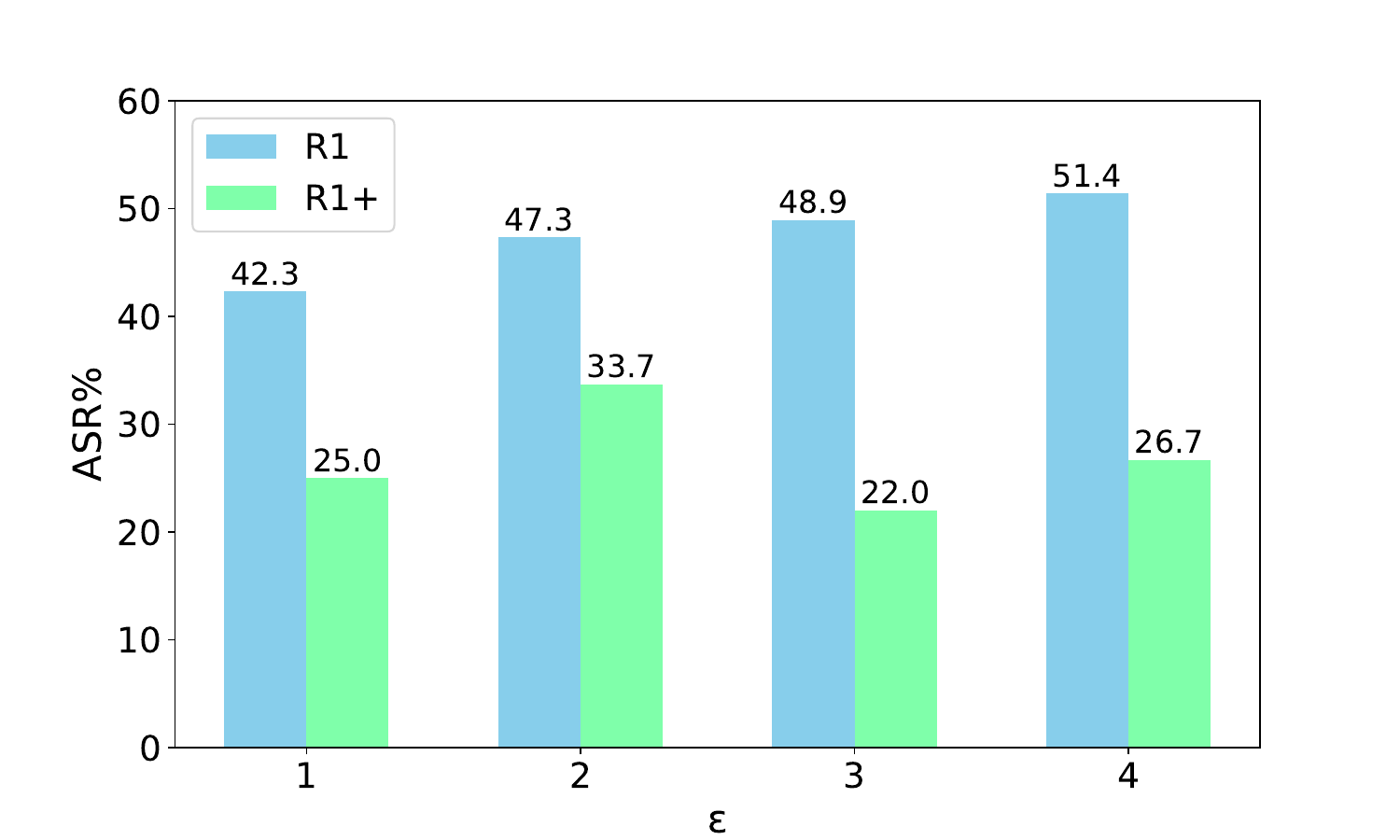}}
        \subfigure[Arch:100-50-10]{
        \label{Fig.sub.2}
        \includegraphics[scale=0.15]{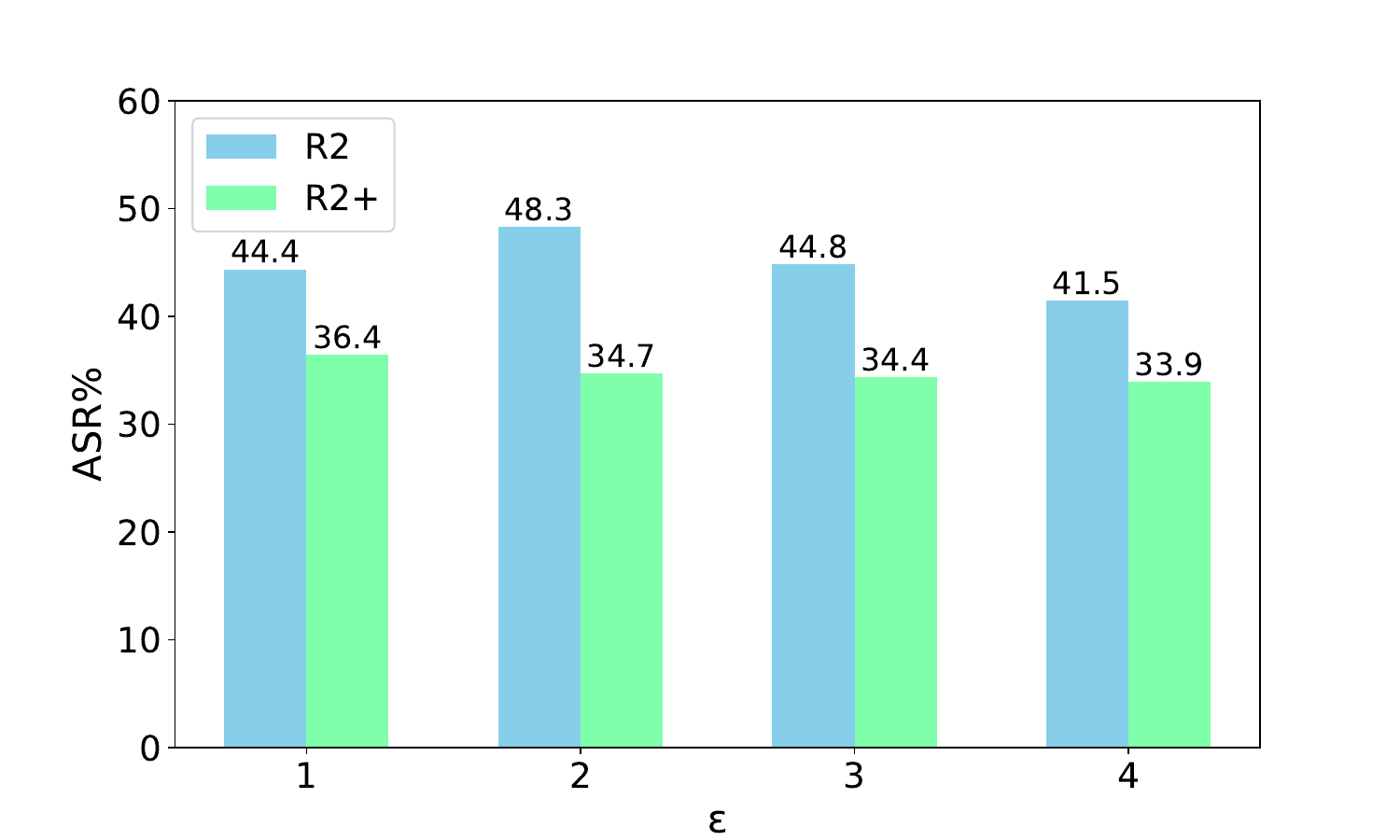}}
        \subfigure[Arch:100-50-32-10]{
        \label{Fig.sub.3}
        \includegraphics[scale=0.15]{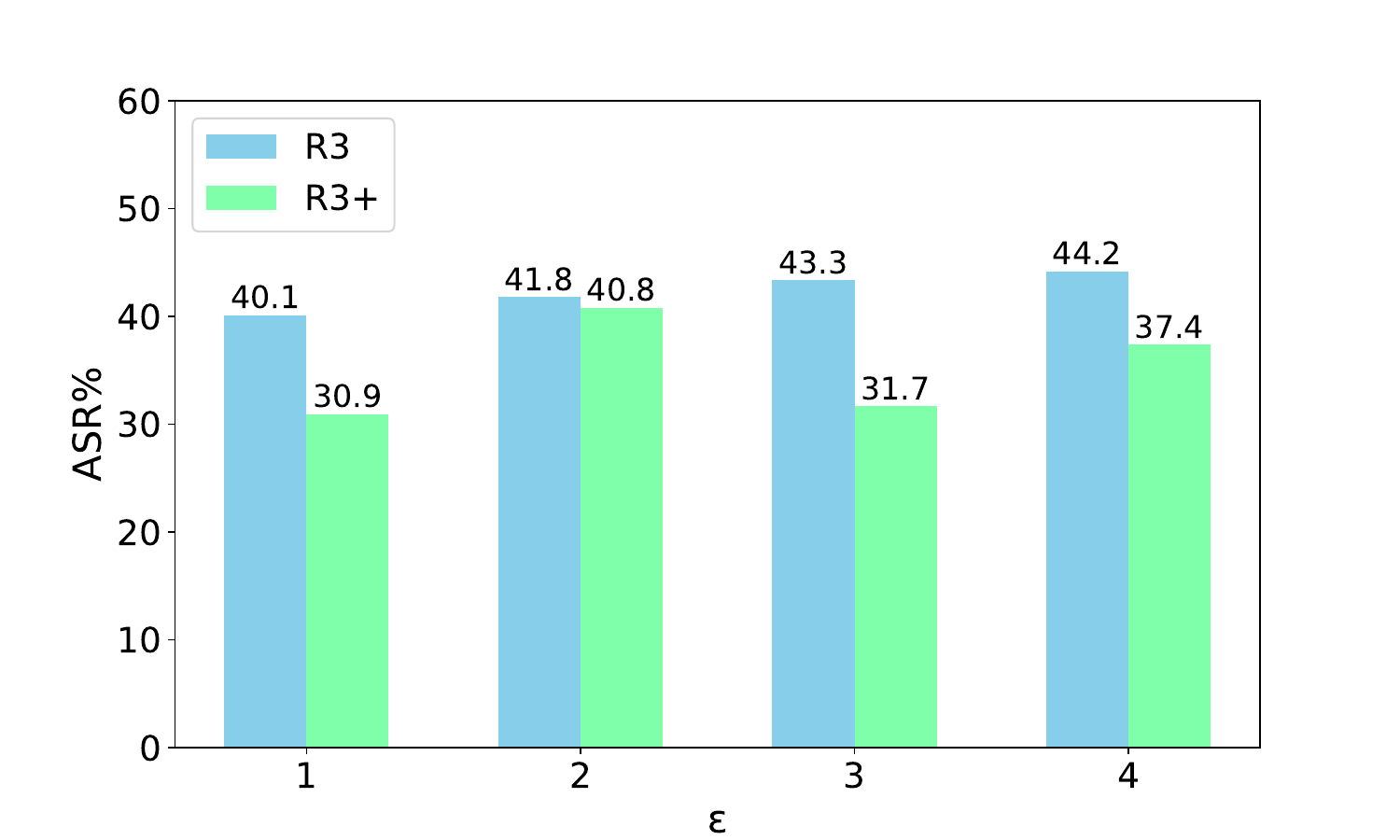}}
        \caption{Results of local robustness.}
        \vspace{-1.0em}
        \label{Fig.ROBUST}
    \end{figure}

    We also take $100$ samples for each case and 
    compare the \as for baselines and their synthesized counterparts.
    The results are shown in Figure \ref{Fig.ROBUST}, 
    blue  bars represent the baselines, 
    while orange bars represent synthesized models.
    We use the sign $+$ to denote the synthesized models.
    Figure~\ref{Fig.sub.1} (resp. Figure~\ref{Fig.sub.2} and Figure~\ref{Fig.sub.3}) depicts
    the percentage of average \as of  \textbf{R1} (resp. \textbf{R2} and \textbf{R3}) and the counterpart \textbf{R1+} (resp. \textbf{R2+} and \textbf{R3+}) (the vertical axis),
    with different $\epsilon$ (1, 2, 3, 4) (the horizontal axis).
    The results demonstrate a decrease in \as by an average of $43.45\%$, $22.12\%$, and $16.95\%$
    for \textbf{R1}, \textbf{R2} and \textbf{R3}, respectively. %

    Whist the models' robustness are enhanced, their accuracy are slightly decreased.
    Table \ref{Table.acc} shows the results of the accuracy of the models, 
    where \accplus represents the average accuracy for synthesized models with the same architectures.
    \begin{table}[ht]
    \vspace{-1.0em}
    \centering
    \caption{The average accuracy of \textbf{R1}-\textbf{R3} and their synthesized models.}
     \setlength{\tabcolsep}{3mm}{
    \begin{tabular}{lccc}
    \toprule
         & \textbf{R1}      & \textbf{R2}      & \textbf{R3}      \\ 
    \midrule
    \acc  & 82.62\% & 84.28\% & 83.50\% \\
    \accplus & 81.33\% & 81.72\% & 78.75\% \\ 
    \bottomrule
    \end{tabular}}
    \label{Table.acc}
    \vspace{-2.5em}
\end{table}
    \subsection{Individual Fairness}
    \label{SubSec: Fair}
    In this section, we investigate the individual fairness w.r.t two sensitive feature, namely 
    sex (Male and Female) and race (White and Black) on the \uci dataset. 

    We consider \textbf{F1}-\textbf{F3} as baselines, and randomly select 1000 entries for both \textbf{F1} and \textbf{F2}, and 200 entries for \textbf{F3} from the training set,
    and then generate proper pairs by modifying the value of the sensitive attribute while keeping all other attributes the same. 
    For example, we modify the value of Male to Female. 
    After forming specifications using the approach mentioned in Section \ref{SubSec: Example_Props} with the pairs, we proceed with the "enhancement" procedure and retraining to obtain the synthesized models.
    We then evaluate the models on the test set by the measuring the fairness score. 
    We count the number of the fair pairs (the pairs only differ in the sensitive attribute, and get the predication): \fairnum, and compute the fairness score, 
        $\frac{\fairnum}{\sizeoftestset}$,
        where \sizeoftestset is the size of the test set.

    \begin{table}[H]
        \centering
        \caption{Results of individual fairness.}
         \setlength{\tabcolsep}{2mm}{
        \begin{tabular}{ccccccc}
        \hline
        \multicolumn{1}{l}{\textbf{Model}} & \multicolumn{1}{l}{\textbf{Feature}} & \multicolumn{1}{l}{\textbf{\acc}} & \multicolumn{1}{l}{\textbf{\accplus}} & \multicolumn{1}{l}{\textbf{\fair}} & \multicolumn{1}{l}{\textbf{\fairplus}} & \multicolumn{1}{l}{\textbf{Synthesis Time(s)}} \\ \hline
        \rowcolor[HTML]{EFEFEF}
        \textbf{F1}                       & sex                         & 80.12\%                 & 74.53\%                  & 92.91\%                  & 99.94\%                   & 241.67                            \\
        \textbf{F1}                       & race                        & 80.12\%                 & 74.54\%                  & 92.92\%                  & 100\%                     & 216.46                            \\
        \rowcolor[HTML]{EFEFEF}
        \textbf{F2}                        & sex                         & 79.88\%                 & 75.71\%                  & 95.68\%                  & 97.83\%                   & 215.61                            \\
        \textbf{F2}                        & race                        & 79.88\%                 & 75.18\%                  & 94.64\%                  & 98.47\%                   & 212.46                            \\
        \rowcolor[HTML]{EFEFEF}
        \textbf{F3}                        & sex                         & 78.13\%                 & 74.48\%                  & 89.67\%                  & 99.83\%                   & 90.39                             \\
        \textbf{F3}                        & race                        & 79.88\%                 & 74.09\%                  & 89.16\%                  & 98.27\%                   & 95.75                             \\ \hline
        \end{tabular}}
        \label{Table.Fair}
    \end{table}

    The results are listed in Table \ref{Table.Fair}, where the baselines and the sensitive attributes shown in Column 1,2. 
    Column 3,4 (\acc/\accplus) demonstrate the accuracy of baselines and synthesized models, and Column 5,6 (\fair/\fairplus) show their fairness scores.
    The figure shows that the all models' individual fairness is significantly improved, with some even reach reaching $100\%$ (Row 2, the fairness score increase from $92.92\%$ to $100\%$). However, the enhancement is accompanied by the accuracy lost,  Column 3,4 show that all models suffer from a certain degree of accuracy decrease.
    Our tool efficiently synthesized the hyper-parameters within a few minutes, as shown in Column 7.

    Furthermore, we examine the ability of our approach on helping determine the architecture of the {\bnn}s.
    For both sex and race, 
    we sample $200$ entries in the training set to generate proper pairs, 
    and formulate the specification 
    without using the bound constraints or fixing the number of block, as follows,
    \vspace{-0.5em}
    \[
        \opF( \bigwedge_i^k(\bm{x}_i=\bm{y}_i)) \wedge
        (\bigwedge_i^k(\bm{x}_i = \rhd^2 \bm{a}_i \wedge \bm{y}_i = \rhd^2 \bm{b}_i) 
        \vee 
        (\bigwedge_i^k(\bm{x}_i = \rhd^3 \bm{a}_i \wedge \bm{y}_i = \rhd^3 \bm{b}_i)))
    \]
    where $(\bm{a}_i,\bm{b}_i)$ is the proper pair, and $k$ is the number of samples.
    The formula indicates the presence of consecutive blocks in the model, with a length of either 2 or 3. 
    For each proper pair $(\bm{a}_i, \bm{b}_i)$, their respective outputs $(\bm{x}_i, \bm{y}_i)$ must be equal.
    
    After synthesizing the partial input-output relation of $f_i$s,we determine the length of the network by selecting the maximum $i$ among $f_i$s.
    The dimensions of the blocks are set to the maximum input and output dimensions in the partial relation obtained for the corresponding $f_i$.
    
    \begin{table}[H]
    \centering
    \caption{The synthesized models which architecture are given by our tool.}
    \setlength{\tabcolsep}{3mm}{
    \begin{tabular}{cccccc}
    \hline
    \textbf{Attr}        & \textbf{Arch} & \textbf{Len} & \textbf{\#Mapping} & \textbf{\acc}    & \textbf{\fair}   \\ \hline
    \rowcolor[HTML]{EFEFEF} 
        sex & 66-10-10-2 & 3 & $1117$  &     $74.38\%$ &  $99.51\%$ \\
        sex & 66-8-2     & 2 & $559$   & $74.69\%$                     & $99.72\%$                       \\
    \rowcolor[HTML]{EFEFEF} 
    race                                            & 66-9-8-2 & 3 & $952$   & $74.38\%$                                           & $94.59\%$                                             \\
    race                                            & 66-8-2   & 2 & $567$   & $74.13\%$                                           & $99.71\%$                                             \\ \hline
    \end{tabular}}
    \label{Table.Fair2}
    \vspace{-2.0em}
    \end{table}
    
    We make a slight adjustment to the synthesis framework,
    when find a group of hyper-parameters, we continue searching for one more feasible group, resulting in two groups of hyper-parameters for sex and race.
    We showcase the synthesized models in Table \ref{Table.Fair2}.
    Column 1 indicates the sensitive attribute of interest,
    and Column 2,3 display the architecture and the length of the {\bnn}s respectively.
    Column 4 shows the number of partial mappings we obtained in the synthesis task.
    Our tool successfully generates  models with varying architectures and high individual fairness, which are
    presented in the Column 5,6 respectively.

 \section{Conclusion}
    \label{Sec: Con}
    In this paper, we have presented an automata-based approach to synthesizing binarized neural networks.
    Specifying {\bnn}s' properties with the designed logic \bltl, and using the tableau-based construction approach, the synthesis framework determine hyper-parameters of {\bnn}s and relations among some parameters, and then we may perform a block-wise training.
    We implemented a prototype tool and the experiments demonstrate the effectiveness of our approach in enhancing the local robustness and individual fairness of {\bnn}s. 
    Although our approach have shown the feasibility of synthesizing trustworthy {\bnn}s, there is still a need to further explore this line of work.
    In the future, beyond the input-output relation of {\bnn}s,
    we plan to focus on specifying properties between the intermediate blocks.
    %Additionally, we aim to explore a more efficient and reasonable way to utilize the obtained hyper-parameters for synthesizing {\bnn}s.
    Additionally, we aim to extend the approach to handle the synthesis task of multi-bits {\qnn}s. 

\subsubsection{Acknowledgements}  This work is partially  supported by the National Key R \& D Program of China (2022YFA1005101), the National Natural Science Foundation of China (61872371, 62072309, 62032024), CAS Project for Young Scientists in Basic Research (YSBR-040), and ISCAS New Cultivation Project (ISCAS-PYFX-202201).

\bibliographystyle{splncs04}
\bibliography{ref}

\appendix

    \section{Appendix}
    \subsection{Proof of Theorem \ref{Thm: Automata-Cons}}
    \label{Appendix: Automata-Cons}
    \begin{proof}
		Still let $f_0, f_1, \cdots, f_{n-1}$
		be the encoding of a \bnn $\network{N}$:
		\par
		\noindent $(\Rightarrow)$:
		Suppose that $\network{N}\in\lang{\automaton{A}_\varphi}$, and $q_0, q_1,\cdots, q_n$ be an accepting run
		of $\automaton{A}_\varphi$ on $\network{N}$, remind that each $q_i$ is a formula set, by induction  on both
		the index (in the backward way) and formulas' structure, we prove the following claim:
		\begin{quotation}
			$\network{N}, i\implies \psi$ for ever $\psi\in q_i$.
		\end{quotation}
		\begin{itemize}
			\item The case is trivial if $\psi=\top$; and $\psi\neq\bot$ since $q_0,q_1,\ldots q_n$ is an accepting run.
			\item If $\psi=\bm{t}_1\sim\bm{t}_2$, we need to distinguish two cases:
			\begin{enumerate}[1)]
				\item When $i=n$, then $\psi\in\cons(q_n)$. We have that $\psi\downarrow$ is evaluated to true,
				because $q_n\in F_\varphi$. Therefore, $\network{N}, n\models\psi$ holds from Thm. \ref{Thm: Intuition}.
				\item If $i<n$, then we have $\psi[f_i]\in q_{i+1}$ according to the automaton construction.
				Inductively, we have $\network{N},i+1\models\psi[f_i]$, and we can then conclude that $\network{N}, n\models\psi$
				according to Thm. \ref{Thm: Intuition}.
			\end{enumerate}
			\item If $\psi=\psi_1\wedge\psi_2$, then we have $\psi_1\in q_i$ and $\psi_2\in q_i$, because $q_i$ is some proper closure.
			Thus $\network{N}, i\models \psi_j$ holds for $j=1,2$ by induction.
			\item  The case of $\psi=\psi_1\vee\psi_2$ are similar to the above.
			\item If $\psi=\psi_1\opU\psi_2$, then $\psi_2\vee(\psi_1\wedge\opX(\psi_1\opU\psi_2))\in q_i$, and subsequently
			either $\psi_2\in q_i$ or both $\psi_1\in q_i$ and $\opX(\psi_1\opU\psi_2)\in q_i$. For the former case,
			we can ensure that $\network{N},i\models\psi_2$.
			For the latter case, since $\opX(\psi_1\opU\psi_2)\in q_i$ we can guarantee that $q_i\not\in F_\varphi$ and  $i<n$,
			subsequently $\psi_1\opU\psi_2\in q_{i+1}$.
			Therefore, we in this case have both $\network{N}, i\models\psi_1$ and $\network{N}, i+1\models\psi_1\opU\psi_2$ by induction.
			\item If $\psi=\psi_1\opR\psi_2$ then $\psi_2\wedge(\psi_1\vee\opXd(\psi_1\opR\psi_2))\in q_i$, which indicates
			that either $\psi_1,\psi_2\in q_i$ or $\psi_2,\opXd(\psi_1\opR\psi_2)\in q_i$. For the former case, we can easily infer
			$\network{N}, i\models\psi_j$ for $j=1,2$ by induction, which implies $\network{N}, i\models\psi$ holds.
			For the latter case, first, we have $\network{N}, i\models\psi_2$; in addition,  we have $\opXd(\psi_1\opR\psi_2)\in q_i$,
			and it could be distinguished by two cases:
			\begin{enumerate}[1)]
				\item $i=n$, then $\network{N}, i\models\opXd(\psi_1\opR\psi_2)$ trivially holds in this case;
				\item $i<n$, then we have $\psi_1\opR\psi_2\in q_{i+1}$, and we also have
				$\network{N}, i\models\opXd(\psi_1\opR\psi_2)$ by induction.
			\end{enumerate}
                \iffalse
			\item If $\psi=\forall\bm{x}\in\domB^k.\phi$, then according to the definition of proper closures, we have
			$\phi^{\bm{x}}_{\bm{b}}\in q_i$ for every $\bm{b}\in\domB^k$. In addition, each $\phi^{\bm{x}}_{\bm{b}}$
			is a closed formula because $\psi$ is. By induction, we have $\network{N}, i\models\phi^{\bm{x}}_{\bm{b}}$,
			and hence $\network{N}, i\models\psi$.
			\item The case of $\psi=\exists\bm{x}\in\domB^k.\phi$ is similar to the above case.
                \fi
		\end{itemize}

		\noindent $(\Leftarrow)$: On the other way round, suppose that $\network{N}\models\varphi$, then let
		$$q_i = \{\psi\in\sub(\varphi)\mid\network{N}, i\models\psi\}$$
		for each $i\leq n$.
		
		We first show that that each $q_i$ is some proper closure
		of some subset of $\sub(\varphi)$. Therefore, we have $q_i\in Q_\varphi$ for each $i$.
		\begin{itemize}
			\item If $\psi_1\wedge\psi_2\in q_i$, then $\network{N}, i\models \psi_j$ for  $j=1,2$, thus
			both $\psi_1$ and $\psi_2$ are in $q_i$.
			\item  If $\psi_1\vee\psi_2\in q_i$, then either $\network{N}, i\models\psi_1$ or $\network{N}, i\models\psi_2$,
			which implies that $\psi_1\in q_i$ or $\psi_2\in q_i$.
			\item Suppose that $\psi_1\opU\psi_2\in q_i$, we can immediately infer that
			$\psi_2\vee(\psi_1\wedge\opX(\psi_1\opU\psi_2))\in q_i$ according to Proposition~\ref{Lem:BLTL_Equiv}.
			\item Similar for the case for the formula $\psi_1\opR\psi_2$.
                \iffalse
			\item If $\forall\bm{x}\in\domB^k.\psi\in q_i$, then it implies that $\network{N},i\models\psi^{\bm{x}}_{\bm{b}}$
			for each $\bm{b}\in\domB^k$, because $\network{N},i\models\forall\bm{x}\in\domB^k.\psi\in q_i$.
			Therefore, we have $\psi^{\bm{x}}_{\bm{b}}\in q_i$ for every proper $\bm{b}$.
			\item The case for $\exists\bm{x}\in\domB^k.\psi$ is similar to the above.
                \fi
		\end{itemize}
		Thus, we can conclude that each $q_i\in Q_\varphi$.  Next, we also need to show that
		$q_{i+1}\in\delta(q_i,f_i)$ for every $i<n$.
		\begin{itemize}
			\item First of all, since $q_i$ consists of closed formulas which is satisfied by $\network{N}$ at step $i$, we can conclude that
			$\bot\not\in q_i$.
			\item For each constraint $\gamma\in\cons(q_i)$, according to Theorem \ref{Thm: Intuition} and the construction,
			we have $\gamma(f_i)\in q_{i+1}$. Therefore, $\cons(q_i)[f_i]\subseteq q_{i+1}$.
			\item For each $\opX\psi\in q_i$, since $i<n$ and $\network{N}, i\models\opX\psi$, then we have $\psi\in q_{i+1}$
			and subsequently $q'_i\subseteq q_{i+1}$ (cf. the automaton construction).
			Likewise, we can also infer that $q''_i\subseteq q_{i+1}$.
			\item Therefore, $q_{i+1}\in\cl(\cons(q_i)[f_i]\cup q'_i\cup q''_i)$, because $q_{i+1}$ must be some proper closure.
		\end{itemize}
		In addition, we have $\varphi\in q_0$ because  $\network{N}\models\varphi$, and we thus have $q_0\in I_\varphi$.
		Moreover, we claim that $q_n\in F_\varphi$ due to $\network{N}, n\models\psi$ for each $\psi\in q_n$, in detail:
		\begin{itemize}
			\item $\gamma\downarrow$ has to evaluated to true, if $\gamma\in\cons(q_n)$, according to Theorem \ref{Thm: Intuition};
			\item $q'_n$ must be $\emptyset$ according to the semantics definition on $\opX$ operator.
		\end{itemize}
		Then, we can conclude that $q_0,q_1,\ldots, q_n$ is an accepting run of $\automaton{A}_\varphi$ on $\network{N}$.
	\end{proof}

    \subsection{Example of the Automata Construction}
    \label{Appendix:example}
        \begin{figure}[t]
        \centering
        \includegraphics[width=0.7\textwidth]{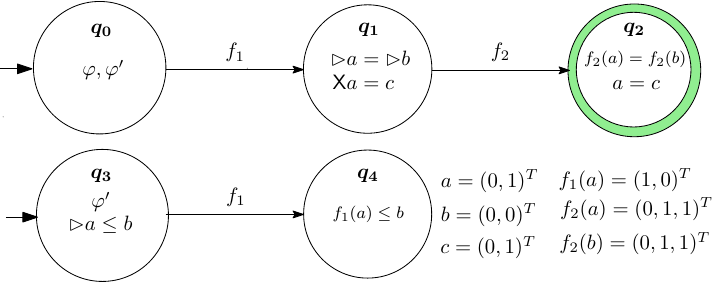}
        \caption{
            The automaton for  $\varphi=\opX \left(\left(\opN \bm{a} = \opN \bm{b}\right) \wedge \opX \left(\bm{a} = \bm{c}\right)\right) 
            \vee  \opN \bm{a} \le \bm{b}$, where
            $\varphi' = \opX \left(\left(\opN \bm{a} = \opN \bm{b}\right) \wedge \opX \left(\bm{a} = \bm{c}\right)\right)$.
%            $f_1((0,1)^\T) = (1,0)^\T$, 
 %           $f_2((0,0)^\T) = f_2((0,1)^\T) = (0,1,1)^T$.
        }
        \label{Fig.automaton}
        \end{figure}
        Consider the \bltl formula 
        $\varphi = \opX \left(\left(\opN \bm{a} = \opN \bm{b}\right) \wedge \opX \left(\bm{a} = \bm{c}\right)\right) 
        \vee  \opN \bm{a} \le \bm{b}$.
        We exemplify the automata construction using $\varphi$.
        The constructed automaton is shown in Fig~\ref{Fig.automaton}, where
        % we illustrate an equivalent automaton for the \bltl formula 
        %R$\varphi = \opX \left(\left(\opN \bm{a} = \opN \bm{b}\right) \wedge \opX \left(\bm{a} = \bm{c}\right)\right) 
        %\vee  \opN \bm{a} \le \bm{b}$.
        $q_1, q_3 \in I_\varphi$, and $q_2 \in F_\varphi$. It is easy to see that
            $q_0,q_1,q_2$ is an accepting path over the input $f_1,f_2$.
        We note that 
        $q_4 \not\in F_\varphi$, since the formula $f_1(a) \le b$ is evaluated false.

    \subsection{Proof of Theorem \ref{Thm: Threshold}}
    \label{Appendix: Threshold}
    \begin{proof}
		Observe that a counter remains unchanged  a \tname{Modal}-node is encountered, thus we
		concentrate to the slicing of the tableau consisting of \tname{Modal}-nodes only.
		First of all, we have the following observations:
		\begin{itemize}
			\item Suppose, $\langle i,\Gamma\rangle$ is a \tname{Modal}-node,
			a (padded or non-padded, but not saturated) constraint $\gamma\in\Gamma$, then we ensure that $\gamma[f_i]$ must occur in the next \tname{Modal}-node whose counter is $i+1$.
			\item Thus, each padded constraint in a \tname{Modal}-node must be of
			the form $\gamma[f_{\ell},f_{\ell+1},\ldots,f_{\ell+t}]$
			where $\gamma$ is a non-padded constraint, and $t<\len(\gamma)$. Namely,
			indices of the layer variables of a padded constraint must be successive.
			For such a constraint, we call $\ell$ and $\ell+t$ the \emph{starting index} and the \emph{ending index}, respectively.
			\item In a same \tname{Modal}-node, all padded constraints share a same
			ending index, but their staring indices may be different.
			Thus, for a \tname{Modal}-node, if the (common) ending index is $w$,
			then each starting index must be less than $w-k$.
			Since we are now concerned about the number of equivalent classes of
			$\cong$, according to Lemma \ref{Lem: Cong}, we may fix the ending index to be $k$, therefore starting indices belong to the set $\{1, 2,\ldots, k\}$.
			\item Call two padded constraints to be \emph{homologous} if they are
			obtained from a same original constraint via applying different layer variable list. Note that homologous is also an equivalent relation,
			and each equivalent class must be of the form
			$\{\gamma[f_\ell, f_{\ell+1},\ldots,f_k] \mid \ell\leq 1\}$ for each $\gamma\in\cons(\bm{\Sigma})\cap\sub(\varphi)$, denoted that set as $H(\gamma)$.
		\end{itemize}
	 Let us now count the upper bound of the equivalence class number of \tname{Modal}-nodes. In a \tname{Modal}-node, we categorize the formulas into two sets: the first consists of constraints, and the second one is
	 constituted with $\opX$- and/or $\opXd$-guarded formulas.
	 \begin{enumerate}[(1)]
	 	\item For each original constraint $\gamma$, the first set may contain a
	 	subset of $H(\gamma)\cup\{\gamma\}$, hence this part has no more than $2^{(k+1)c}$ possibles.
	 	\item In a \tname{Modal}-node, each $\opX$-guarded (resp. $\opXd$-guarded) formula corresponds a subformula of $\varphi$, whose
	 	out-most operator is either $\opX$ (resp. $\opXd$) or $\opU$ (resp. $\opR$). Thus, the number of such formulas occurring in the node is not
	 	more than $p$, and such part yields not more than $2^p$ subsets.
	 \end{enumerate}
     As a result, once the counter becomes $2^{(k+1)c+p}+1$, we may declare that
     some isomorphic \tname{Model}-node already exists in the current path, hence
     it could be a candidate value of threshold.
	\end{proof}
\end{document}